\documentclass{article} % For LaTeX2e
\usepackage{iclr2027_conference,times}

\usepackage{amsthm}

\newtheorem{proposition}{Proposition}

\newtheorem{theorem}{Theorem}
\newtheorem{assumption}{Assumption}

\newtheorem{lemma}{Lemma}

\usepackage{amsmath,amsfonts,bm}

\def\eqref#1{equation~\ref{#1}}
\def\1{\bm{1}}

\DeclareMathAlphabet{\mathsfit}{\encodingdefault}{\sfdefault}{m}{sl}
\SetMathAlphabet{\mathsfit}{bold}{\encodingdefault}{\sfdefault}{bx}{n}

\newcommand{\E}{\mathbb{E}}

\DeclareMathOperator*{\argmax}{arg\,max}
\DeclareMathOperator*{\argmin}{arg\,min}

\usepackage{hyperref}
\usepackage{url}
\usepackage{graphicx}
\usepackage{booktabs}
\usepackage{multirow}
\usepackage{float} 
\usepackage{caption}
\usepackage{algorithm}
\usepackage{algpseudocode}

\title{Learning Samples Importance: Parameterizing Dual Variables in Everywhere Learning.}

\author{Ignacio Boero, Jonathan Nixon \& Alejandro Ribeiro  \\
University of Pennsylvania\\
\{iboero, jnixon27, aribeiro\}@engineering.upenn.edu}

\iclrfinalcopy % Uncomment for camera-ready version, but NOT for submission.

\begin{document}

\maketitle
\lhead{Preprint.}

\begin{abstract}
Everywhere learning provides a principled framework for training AI models under constraints that must hold throughout the data distribution. In the dual domain, these pointwise constraints give rise to functional dual variables.
In this work, we propose to learn these dual variables, motivated by the fact that their values encode useful information about the underlying constrained problem. By representing the dual variable as a parametric function of each sample, we enable the learned multiplier to be evaluated on new, unseen samples. This contrasts with standard empirical dual formulations, which assign an independent multiplier to each training sample.
We characterize the error in the recovered primal solution induced by restricting the dual variable to a parametric function class and show that it is controlled by how well this class approximates the optimal statistical multiplier. Moreover, we show that the learned parametric multiplier retains the sensitivity interpretation of the optimal statistical multiplier, yielding approximate sensitivity guarantees that extend beyond the samples used for training.
We empirically validate our theory across a variety of everywhere learning tasks, showing that the resulting constrained problems can be solved efficiently and that the learned dual variables provide meaningful representations of sample-level sensitivity.
\end{abstract}

\section{Introduction}

Constrained learning provides a natural framework for training AI systems subject to multiple  requirements~\citep{pmlr-v54-zafar17a,JMLR:v20:18-616,chamon2021probablyapproximatelycorrectconstrained}. Expressing requirements as constraints prevents improvements in one criterion from compensating for violations of another. Everywhere learning proposes to impose requirements pointwise over the data distribution instead of in expectation, preventing imbalances in performance across data regions~\citep{boero2026everywhere,hounie2026every}. Solving the resulting problem naturally leads to the dual domain, where the dual multipliers are themselves functions of the data. 
In classical optimization, optimal dual multipliers carry a sensitivity interpretation, quantifying how changes to individual constraints affect the optimal value~\citep{boyd2004convex,doi:10.1137/1035044,doi:10.1137/1.9781611976595}.
This motivates our central question of whether these dual multipliers can be learned.
\begin{itemize}
\item[(C1)] We propose to learn dual multipliers by parameterizing them as a function of the samples.
\end{itemize}
Rather than assigning an independent multiplier to every training sample, we optimize over a parametric family $\Lambda_\Phi$, allowing the learned multiplier to be evaluated on unseen samples. This, however, restricts the original functional dual domain to $\Lambda_\Phi$, and it is therefore not immediate that the resulting dual problem still recovers a suitable primal solution.
\begin{itemize}
\item[(C2)] We provide near-feasibility and near-optimality guarantees for the primal solution recovered from the parametric dual problem.
\end{itemize}
We show that the primal solution induced by the learned parametric dual remains approximately feasible and optimal. In particular, the additional error introduced by parameterizing the dual domain is controlled by how well $\Lambda_\Phi$ approximates an optimal multiplier. Our goal, however, is not only to recover a good primal model, but also to preserve the information carried by the multipliers themselves.
\begin{itemize}
\item[(C3)] We show that the learned multiplier retains the sensitivity role of the optimal dual variable.
\end{itemize}
The learned multiplier continues to approximately quantify how perturbing individual constraints can affect the optimal value of the problem. Unlike primal recovery, extending this interpretation to unseen samples additionally requires the dual class to generalize, revealing a natural tradeoff: richer parameterizations can better approximate the optimal multiplier but may incur larger statistical error.
We validate these results across four diverse machine-learning problems spanning image classification, agentic workflows, language-model unlearning, and optimal power flow.
\begin{itemize}
\item[(C4)] We show empirically that parametric duals preserve primal performance while learning transferable sample-level sensitivity.
\end{itemize}
Across tasks, expressive dual parameterizations approach the performance of sample-wise multipliers, remain effective as the number of constraints grows, and accurately generalize multiplier information to unseen samples. 

%We further show that these predictions identify influential constraints and can be used to prioritize new samples for improving the constrained solution.

\section{Everywhere Learning}
\label{sec:every_learn}

Let \(\{\mathcal D_i\}_{i=0}^m\) be probability distributions over input-output pairs \((x,y) \in \mathcal X \times \mathcal Y\). Given bounded loss functions \(\ell_i:\mathcal Y\times\mathcal Y\to[0,B]\), for \(i=0,\ldots,m\), and constraint thresholds \(c_i\in\mathbb R\), for \(i=1,\ldots,m\), the \textit{everywhere learning problem} seeks the  parametric predictor \(f_\theta:\mathcal X\to\mathcal Y\) in the hypothesis class $\mathcal H := \{f_\theta:\theta\in\Theta\subseteq\mathbb R^p\}$ that solves the constrained problem:
\begin{equation}\label{P:primal_statistical}\tag{P}
\begin{aligned}
	f_\theta^* = \argmin_{f_\theta \in \mathcal{H}}&
		&&\E_{(x,y) \sim \mathcal D_0} \!\Big[ \ell_0\big( f_{\theta}(x),y \big) \Big]
	\\
	\text{s. to}& &&  \ell_i\big( f_{\theta}(x),y \big)  \leq c_i \quad \mathcal D_i\text{-a.e.}
\end{aligned}
\end{equation}

The everywhere learning problem~\ref{P:primal_statistical} provides a formulation for learning predictors subject to multiple requirements. Expressing these requirements as constraints makes their prescribed tolerances explicit and prevents improvements in one criterion from compensating for violations of another, as may occur when minimizing weighted combinations of losses. Moreover, the almost-everywhere constraints require each requirement to hold throughout its corresponding data distribution. This prevents good performance in some regions from compensating for violations in others, as can occur when constraints are imposed only in expectation. 

Henceforth, we consider a single constraint $m=1$, set its tolerance $c=0$ without loss of generality, denote $z=(x,y) \in \mathcal{Z} := \mathcal{X}\times \mathcal{Y}$ be a compact representation of the input-output pairs, and define $\ell(f_\theta(x),y)=\ell(f_\theta,z)$, for notational convenience. 

\subsection{Everywhere Dual Problem}

A standard approach when solving constrained problems is to turn to the dual domain~\cite{}.
Since Problem~\ref{P:primal_statistical} imposes a constraint pointwise at every \(z\), its dual variable
\(\lambda:\mathcal{Z}\to\mathbb{R}_+\) is itself a nonnegative function, where \(\lambda(z)\) weights the constraint imposed at point \(z\).
We restrict these functions to the space of nonnegative functions integrable with respect to the constraint distribution, that is,
\(\lambda\in L_1^+(\mathcal{D})\).
The Lagrangian function
\(\mathcal{L}:\mathcal{H}\times L_1^+(\mathcal{D})\to\mathbb{R}\)
aggregates the objective value with the constraint loss weighted pointwise by the corresponding dual variable,
\begin{align}
\label{eq:lagrangian}
\mathcal{L}(f_\theta,\lambda)
:=
\mathbb{E}_{z\sim\mathcal{D}_0}
\left[
    \ell_0(z,f_\theta)
\right]
+
\mathbb{E}_{z\sim\mathcal{D}}
\left[
    \lambda(z)\ell(z,f_\theta)
\right].
\end{align}
For a fixed multiplier \(\lambda\), the Lagrangian function reduces the constrained problem to an unconstrained optimization over the predictor. The function that takes a multiplier and returns the Lagrangian minimization is denoted the \emph{dual function},
\begin{align}
\label{eq:dual_function}
g(\lambda)
:=
\mathcal{L}\bigl(f_\theta(\lambda),\lambda\bigr)
=
\min_{f_\theta\in\mathcal{H}}
\mathcal{L}(f_\theta,\lambda).
\end{align}
The multiplier \(\lambda\) therefore determines how the pointwise constraints are weighted when optimizing the predictor.
Ideally, we seek a weighting for which the resulting Lagrangian minimizer \(f_\theta(\lambda)\) is feasible for Problem~\ref{P:primal_statistical} while achieving a low objective value.
A principled approach for searching over such weightings is given by the dual problem, which maximizes the dual function over all admissible nonnegative multipliers,
\begin{align}
\label{P:dual_statistical}
\tag{D}
\lambda^\star
\in
\arg\max_{\lambda\in L_1^+(\mathcal{D})}
g(\lambda),
\end{align}
and yields $D^\star = g(\lambda^\star)$ as its optimal value.
Previous work establishes the connection between Problem~\ref{P:primal_statistical} and its dual formulation.
For convex problems, the primal and dual optimal values coincide under mild conditions.
For the nonconvex parametrizations commonly used in deep learning, this equivalence need not hold exactly; nevertheless, the resulting duality gap can be controlled by the richness of the predictor class~\cite{boero2026everywhere}.
Thus, the dual formulation provides a principled approach for studying and solving the everywhere learning problem.

Beyond its role in solving the constrained problem, the optimal multiplier \(\lambda^\star(z)\) provides a sample-level measure of how strongly each constraint shapes the learned predictor.
In the next section, we study the conditions under which the optimal multiplier $\lambda^*(z)$ can be learned from a parametric family $\lambda_\phi(z)$ while simultaneously solving the constrained problem from samples.
This differs from the standard empirical dual formulation, which assigns an independent multiplier to each training sample and thus can not provide dual values for new samples \cite{boero2026everywhere}.
%
%But first, we illustrate three applications in which the constrained formulation arises naturally, and access to the optimal multiplier is useful.
%\input{02_problem_formulation/c_apps}
\label{sec:learn_param}

\section{Parametric Empirical Dual Problem}
\label{sec:param}
As in standard machine learning settings, the distributions $D_0$ and $D$ are unknown and can only be accessed through samples. We therefore replace the expectations in \eqref{eq:lagrangian} by sample averages over datasets $S_0 = \{z_n\}_{n=1}^{N_0}$ and $S=\{z_n\}_{n=1}^N$, yielding the empirical Lagrangian
\begin{equation}
\label{eq:emp_lagrangian}
\hat L(f_\theta,\lambda)
:=
\frac{1}{N_0} \sum_{z_n \in S_0}
    \ell_0(z_n,f_\theta)
+
\frac{1}{N}\sum_{z_n \in S}
    \lambda(z_n)\ell(z_n,f_\theta).
\end{equation}
As in the statistical setting, fixing the dual multiplier $\lambda$ and minimizing the empirical Lagrangian over the hypothesis class defines the empirical dual function
\begin{equation}
\hat{g}(\lambda)
=
\min_{f_\theta \in \mathcal{H}}
\hat{L}(f_\theta,\lambda).
\label{eq:empirical-dual-function}
\end{equation}

The corresponding empirical dual problem would maximize $\hat g(\lambda)$ over nonnegative multiplier functions. Direct optimization over the infinite-dimensional space $L_1^+(\mathcal D)$, however, is generally intractable. To obtain a finite-dimensional formulation, we propose learning the dual multiplier within a bounded parametric hypothesis class $\Lambda_\Phi$,
\begin{equation}
\Lambda_\Phi
=
\bigl\{
\lambda_\phi:\mathcal Z\to\mathbb R_+
\;\bigm|\;
\lambda_\phi\in L_1^+(\mathcal D),\
\phi\in\Phi\subseteq\mathbb R^q,
\|\lambda_\phi\|_\infty\leq\Gamma
\bigr\}.
\label{eq:parametric-family}
\end{equation}
We then define the \emph{empirical parametric dual problem} by maximizing the empirical dual function over $\Lambda_\Phi$,
\begin{equation}
\tag{$\widehat{\textup{D}}_\Phi$}
\hat\lambda^\star_\Phi
\in
\argmax_{\lambda_\phi\in\Lambda_\Phi}
\hat g(\lambda_\phi),
\label{eq:parametrized-dual}
\end{equation}
with optimal value $\hat D^\star_\Phi
=\hat g(\hat\lambda^\star_\Phi)$.

This approach contrasts with the standard construction of an \emph{empirical dual variable} $\hat\lambda\in\mathbb R_+^N$, in which one multiplier is assigned independently to each training sample. Such a representation is defined only on the observed dataset and therefore provides no mechanism for assigning multipliers to unseen samples. By contrast, each $\lambda_\phi\in\Lambda_\Phi$ defines a multiplier function over the entire domain $\mathcal Z$, allowing the learned dual variable to be evaluated beyond the training set.

We remark that replacing expectations by sample averages and restricting the dual domain to the parametric class $\Lambda_\Phi$ constitute two approximations required to obtain a tractable formulation of the statistical dual problem \ref{P:dual_statistical}. These modifications, however, imply that problem \ref{eq:parametrized-dual} is no longer the dual of the original statistical primal problem \ref{P:primal_statistical}.
In this section, we show that both this discrepancy and the error in interpreting $\hat\lambda^\star_\Phi$ as the sensitivity of the statistical problem vanish as the parametric family better approximates the optimal multiplier $\lambda^*$ and the number of samples increases.

\subsection{Primal recovery for \ref{eq:parametrized-dual}}
\label{sec:primal-recovery}

The parametrized empirical dual problem~\eqref{eq:parametrized-dual} introduces two approximations with respect to the statistical dual problem~\eqref{P:dual_statistical}. Therefore, to characterize the effect of these approximations, we require conditions that control each of them. The sampling error is governed by uniform convergence of the objective and constraint losses, while the dual-parameterization error is determined by how well $\Lambda_\Phi$ can approximate an optimal statistical multiplier. We formalize these two requirements next.

\begin{assumption}[Uniform convergence]
\label{as:uniform-convergence}
There exist functions $\zeta_0(N_0,\delta)$ and
$\zeta(N,\delta)$, monotonically decreasing to zero in $N_0$ and
$N$, respectively, such that, with probability at least $1-\delta$,
\begin{align}
    \sup_{f_\theta\in\mathcal H}
    \left|
        \mathbb E_{D_0}
        [\ell_0(f_\theta,z)]
        -
        \frac{1}{N_0}
        \sum_{z_n\in S_0}
        \ell_0(f_\theta,z_n)
    \right|
    &\leq
    \zeta_0(N_0,\delta),
    \\
    \sup_{f_\theta\in\mathcal H}
    \left|
        \mathbb E_D
        [\ell(f_\theta,z)]
        -
        \frac{1}{N}
        \sum_{z_n\in S}
        \ell(f_\theta,z_n)
    \right|
    &\leq
    \zeta(N,\delta).
\end{align}
\end{assumption}

\begin{assumption}[Approximation of the optimal multiplier]
\label{as:dual-approximation}
Let $\lambda^\star$ be an optimal multiplier of the statistical
dual problem. The approximation error induced by $\Lambda_\Phi$ is
\begin{equation}
\label{eq:dual-approximation}
    \nu_\Phi
    :=
    \inf_{\lambda_\phi\in\Lambda_\Phi}
    \|
        \lambda_\phi-\lambda^\star
    \|_{L_1(D)}.
\end{equation}
\end{assumption}

Assumption~\ref{as:uniform-convergence} is a standard learnability condition controlling the discrepancy between population and empirical losses uniformly over $\mathcal H$. An example of obtaining those rates from controlled Rademacher complexity of $\mathcal{H}$ is discussed in Appendix~\ref{apx:rademacher}. Assumption~\ref{as:dual-approximation}, on the other hand, measures how well the parametric dual family can approximate an optimal statistical multiplier. The assumption itself is mild, as the approximation error is not required to be small a priori; rather, its value explicitly quantifies the price of restricting the functional dual domain to $\Lambda_\Phi$. 

Under these assumptions we can control the discrepancy between the optimal statistical and parametrized empirical dual values.

\begin{lemma}[Dual learnability]
\label{lem:dual-value}
Let Assumptions~\ref{as:uniform-convergence} and
\ref{as:dual-approximation} hold. Recall that, by definition of
$\Lambda_\Phi$, every $\lambda_\phi\in\Lambda_\Phi$ satisfies
$\|\lambda_\phi\|_\infty\leq\Gamma$. Then, with probability at least
$1-3\delta$,
\begin{equation}
\label{eq:dual-value-absolute}
    \left|
        D^\star-\widehat D_\Phi^\star
    \right|
    \leq
    B\nu_\Phi
    +
    \zeta_0(N_0,\delta)
    +
    \Gamma\zeta(N,\delta)
    =:
    \Delta_D.
\end{equation}
\end{lemma}
\begin{proof}
    See Appendix~\ref{apx:dual-learn}
\end{proof}

Lemma~\ref{lem:dual-value} separates the dual approximation error into two components. The terms $\zeta_0$ and $\Gamma\zeta$ arise from replacing population quantities by finite-sample estimates, whereas $B\nu_\Phi$ is the additional error introduced by restricting the dual domain to $\Lambda_\Phi$. 
Closeness of the optimal dual values, however, does not by itself guarantee that a primal minimizer is an approximate solution of \ref{P:primal_statistical}. To recover such guarantees, we require extra assumption on the structure of the problem \ref{P:primal_statistical}, which ensures that it can be well approximated through the dual problem.

\begin{assumption}[Convex losses and curvature]
\label{ass:convex_out}
For every $y\in\mathcal Y$, the losses $\ell(\cdot,y)$ and
$\ell_0(\cdot,y)$ are convex, and $\ell_0(\cdot,y)$ is
$\mu$-strongly convex in its prediction argument. \footnote{For the standard-Lagrangian argument, the strong-convexity
condition is understood in the prediction norm controlling the
constraint losses. This requirement is avoided by the augmented
Lagrangian used in our experiments, whose dual smoothness follows
directly from augmentation; see Appendix~\ref{app:augmented}.}
\end{assumption}

\begin{assumption}[Lipschitz losses]
\label{ass:lip_out}
For every $y\in\mathcal Y$, the losses
$\ell(\cdot,y)$ and $\ell_0(\cdot,y)$ are $L$-Lipschitz continuous
in their prediction argument, that is,
\begin{equation}
    \left|
        \ell_i(\hat y_1,y)
        -
        \ell_i(\hat y_2,y)
    \right|
    \leq
    L
    \|\hat y_1-\hat y_2\|.
\end{equation}
\end{assumption}

\begin{assumption}[Approximate convexity of the model class]
\label{as:approx-convexity}
The hypothesis class $\mathcal H$ approximates its closed convex hull.
Specifically, letting
\begin{equation}
    \overline{\mathcal H}
    :=
    \overline{\operatorname{conv}}(\mathcal H),
\end{equation}
for every $f\in\overline{\mathcal H}$ there exists
$f_\theta\in\mathcal H$ such that
\begin{equation}
\label{eq:approx-convexity-function}
    \|f_\theta-f\|_\infty
    \leq
    \nu_{\mathcal H}.
\end{equation}
\end{assumption}

\begin{assumption}[Strict Slater condition]
\label{as:slater}
There exists $\bar f_\theta\in\mathcal H$ and $s>0$ such that
\begin{equation}
\label{eq:strict-slater}
    \ell(\bar f_\theta,z)
    +
    L\nu_{\mathcal H}
    \leq
    -s,
    \qquad
    D\text{-a.e.}
\end{equation}
\end{assumption}

Assumptions~\ref{ass:convex_out} and~\ref{ass:lip_out} are standard regularity conditions and are satisfied by common losses on bounded prediction domains, including the squared loss. %
Assumption~\ref{as:approx-convexity} measures how closely the original hypothesis class can represent predictors in its convex hull. It holds with $\nu_{\mathcal H}=0$ for convex hypothesis classes. For neural networks, finite combinations of predictors can be represented by increasing model width and combining the corresponding subnetworks at the output layer, consistent with the view of wide neural networks as linear combinations of learned features~\citep{NIPS2005_0fc170ec}. Thus, $\nu_{\mathcal H}$ can be interpreted as a capacity-dependent approximation error which can decrease as the model class becomes richer.
Finally, Assumption~\ref{as:slater} is a strengthened Slater condition. It requires the existence of a model which can satisfy the constraints with a margin $L\nu_{\mathcal H}$, and is a mild strengthening of requiring the existence of a feasible predictor.

Under these assumptions, the primal minimizer associated with the learned parametric multiplier is both approximately feasible and approximately optimal for the original statistical problem \ref{P:primal_statistical}.

\begin{theorem}[Statistical primal recovery]
\label{thm:primal-recovery}
Let Assumptions~\ref{as:uniform-convergence}--\ref{as:slater} hold,
and let $\widehat f_{\theta_\Phi}$ be a primal minimizer associated
with the optimal parametrized empirical multiplier,
\begin{equation}
\label{eq:recovered-primal}
    \widehat f_{\theta_\Phi}
    \in
    \arg\min_{f_\theta\in\mathcal H}
    \widehat L
    (f_\theta,\widehat\lambda_\Phi^\star).
\end{equation}
Then, with probability at least $1-4\delta$,
$\widehat f_{\theta_\Phi}$ is near-feasible and near-optimal for
problem~\eqref{P:primal_statistical}. In particular,
\begin{align}
&\E_{z\sim D}
\left[
    [\ell(\widehat f_{\theta_\Phi},z)]_+
\right]
\leq\;
\sqrt{
    \frac{2L^2}{\mu}
    \left[
        \Delta_D
        +
        \zeta_0
        +
        L(1+\Gamma)\nu_{\mathcal H}
    \right]
}
+
\sqrt{
    \frac{
        2L^3(1+\Gamma)\nu_{\mathcal H}
    }{\mu}
}
+
\zeta
=:
\Delta_F ,
\label{eq:statistical-feasibility}
\\
&\left|
    \E_{z\sim D_0}
    \left[
        \ell_0(\widehat f_{\theta_\Phi},z)
    \right]
    -
    P^\star
\right|
\leq\;
\zeta_0
+
\Delta_D
+
L
\left(
    1+
    \frac{B_0}{s}
\right)
\nu_{\mathcal H}
+
\left(
    \Gamma
    +
    \frac{B\mu}{L^2}
\right)
\Delta_F
=:
\Delta_O.
\label{eq:statistical-objective}
\end{align}
\end{theorem}

\begin{proof}
See Appendix~\ref{apx:primal-recovery}.
\end{proof}

Theorem~\ref{thm:primal-recovery} shows that solving the parametrized
empirical dual recovers a primal model whose statistical constraint
violation and objective error are proportional to either the statistical,
dual-parameterization, or model-approximation errors. In particular, whenever these three approximations simultaneously vanish, that is $\zeta_0,\zeta,\nu_\Phi,\nu_{\mathcal H}\to0$, both the feasibility bound $\Delta_F$ and the optimality bound $\Delta_O$ vanish to zero.
Moreover, we remark that the explicit
approximation penalty introduced by restricting the functional dual
domain to $\Lambda_\Phi$ is only the error $B\nu_\Phi$ from approximating the optimal multiplier. This error propagates to
primal feasibility and optimality only through $\Delta_D$. The
remaining terms arise from finite-sample estimation and from recovering a primal solution in the original, possibly non-convex, hypothesis class. This implies that as the parametric dual family becomes expressive enough to approximate the optimal statistical multiplier, its primal recovery guarantees approach those obtained from the corresponding unrestricted empirical dual formulation.

\subsection{Sensitivity of the learned parametric multiplier}
\label{sec:sensitivity}

The primal recovery guarantees of Theorem~\ref{thm:primal-recovery} establish that the parametrized empirical dual remains a meaningful surrogate for the statistical primal problem~\ref{P:primal_statistical}. Our main motivation for learning a parametric multiplier, however, is to retain the sensitivity interpretation of the optimal statistical multiplier $\lambda^\star$ while being able to evaluate it on new, unseen samples. We formalize this interpretation next.

Let $u:\mathcal Z\to\mathbb R$ assign a pointwise perturbation to the constraint threshold. We define the perturbation function $P(u)$ as the optimal value of the perturbed primal problem,
\begin{equation}
\label{P:primal_perturbed}
\tag{$\text{P}_u$}
\begin{aligned}
P(u)=\min_{f_\theta\in\mathcal H}\quad &\E_{z\sim\mathcal D_0}\!\left[\ell_0(f_\theta,z)\right]\\
\text{s. to}\quad &\ell(f_\theta,z)\leq u(z)\qquad \mathcal D\text{-a.e.}
\end{aligned}
\end{equation}
The perturbation $u$ changes the original constraint threshold from $0$ to $u(z)$ at each sample. Since larger values of $u$ relax the constraints, $P(u)$ is nonincreasing with respect to the pointwise order. The optimal statistical multiplier characterizes the sensitivity of this value function.

\begin{lemma}[Sensitivity of the optimal multiplier]
\label{lemma:dual-are-subg}
Let $\lambda^\star$ be the optimal statistical dual multiplier. Then $-\lambda^\star$ is a $\Delta_{\rm gap}$-subgradient of $P(0) = P^*$. 

That is, for every perturbation $u$
\begin{equation}
P(u)\geq P^*-\mathbb E_{z\sim\mathcal D}\!\left[\lambda^\star(z)u(z)\right]-\Delta_{\rm gap},
\end{equation}
where $\Delta_{\rm gap}:=|P^\star-D^\star|$ is the duality gap.
\end{lemma}

Lemma~\ref{lemma:dual-are-subg} shows that the multiplier $\lambda^\star(z)$ quantifies how costly it can be to tighten a constraint associated with a particular sample. In particular, a larger multiplier will potentially translate into a larger increase  in the optimal objective when that constraint is tightened, whereas a zero multiplier implies no decrease from relaxing it will happen to the objective. As a learned parametric multiplier can be evaluated on unseen samples, it provides a measure of how much enforcing the constraint on an unseen sample---for instance, by including that sample in the training set---can affect the optimal objective.

To transfer this interpretation to the learned multiplier we require also that the parametric dual class generalizes.

\begin{assumption}[Dual class complexity]
\label{as:dual-uc}
There exists a function $\zeta_\Lambda(N,\delta)$, monotonically decreasing to zero in $N$, such that, with probability at least $1-\delta$,
\begin{equation}
\sup_{\lambda_\phi\in\Lambda_\Phi}\left|\mathbb E_D[\lambda_\phi(z)]-\frac{1}{N}\sum_{z_n\in S}\lambda_\phi(z_n)\right|\leq\zeta_\Lambda(N,\delta).
\end{equation}
\end{assumption}

Assumption~\ref{as:dual-uc} mirrors Assumption~\ref{as:uniform-convergence} for the dual class. As before, we assume that $\zeta_\Lambda(N,\delta)$ arises from vanishing Rademacher complexity and refer to Appendix~\ref{apx:rademacher} for the details. Under the same assumption requiered primal recover, with the dual complexity bound, we can characterize the sensitivity of the learned multiplier.

\begin{theorem}[Empirical approximate sensitivity]
\label{thm:empirical-sensitivity}
Let assumptions~\ref{as:uniform-convergence}-~\ref{as:dual-uc} hold. Then, with probability at least $1-4\delta$, for every $u\in L_1(D)$ it holds that
\begin{equation}
P(u)\geq P(0)-\mathbb E_{z\sim D}\!\left[\widehat\lambda_\Phi^\star(z)u(z)\right]-\left[L\left(1+\frac{B_0}{s}\right)\nu_{\mathcal H}+B\nu_\Phi+2\zeta_0+2\Gamma\zeta+B\zeta_\Lambda\right].
\label{eq:empirical-sensitivity}
\end{equation}
\end{theorem}
\begin{proof}
See Appendix~\ref{apx:sensitivity-thm}.
\end{proof}

Theorem~\ref{thm:empirical-sensitivity} shows that the learned multiplier is also an approximate subgradient of the perturbation function. In this case, however, additional terms appear in the subgradient approximation. Specifically, the term involving $\nu_{\mathcal H}$ upper-bounds the duality gap, while $B\nu_\Phi$ captures the approximation error induced by the parametric dual family. The remaining terms arise from finite-sample estimation, with $B\zeta_\Lambda$ accounting specifically for the statistical complexity of the learned multiplier class.

The result also reveals a tradeoff in the choice of $\Lambda_\Phi$. A richer dual class may reduce the approximation error $\nu_\Phi$, but at the cost of greater statistical complexity and, consequently, a larger $\zeta_\Lambda$. Sensitivity learning therefore favors a class that is expressive enough to approximate $\lambda^\star$ while remaining statistically learnable. This tradeoff does not arise in the same form in Theorem~\ref{thm:primal-recovery}: primal recovery depends on the dual class through $\nu_\Phi$, whereas the additional complexity term $\zeta_\Lambda$ appears only when the learned multiplier itself must generalize to unseen samples.

In the next section, we empirically study the phenomena predicted by the theory and study how the expressivity of the parametric dual family affects the sensitivity estimation. Our experiments use an augmented Lagrangian formulation, for which the theoretical results admit analogous extensions. To keep the main development concise, we defer these extensions to Appendix~\ref{app:augmented}.
\section{Numerical Experiments}
\label{sec:num-exp}

In this section, we evaluate the proposed parametric dual formulation across four diverse learning problems. We consider image classification, agentic workflow prediction, language-model unlearning, and optimal power flow. We next provide a brief description of each task and its evaluation protocol; complete implementation details are provided in Appendix~\ref{apx:d_add_res}.

\paragraph{MNIST: Image classification.}
We consider robust classification on MNIST~\citep{726791}, with $60{,}000$ handwritten digit images. Each sample is constrained to be correctly classified with a prescribed margin, while the primal objective minimizes the squared $\ell_2$ norm of the model parameters. Both the primal predictor and dual multiplier are parameterized by CNNs.

\paragraph{FLORA-Bench: Agentic workflow prediction.}
We use the Coding-GD dataset from FLORA-Bench~\citep{zhang2025gnnspredictorsagenticworkflow}, where a graph neural network predicts whether a multi-LM workflow solves a coding task. The primal objective optimizes task-wise ranking, while pointwise binary cross-entropy constraints enforce accurate classification for every workflow--task pair. Both the primal predictor and the dual multiplier are parameterized by lightweight graph neural networks with a shared backbone.

\paragraph{TOFU: Language-model unlearning.}
We study constrained language-model unlearning on TOFU~\citep{maini2024tofutaskfictitiousunlearning} using the OpenUnlearning framework~\citep{dorna2026openunlearning}. The primal objective preserves model behavior on a retain set, while pointwise constraints enforce forgetting on each example in the forget set. The primal model is a pretrained language model, while the dual multiplier is parameterized by a lightweight neural network over sample representations.

\paragraph{IEEE 30: Optimal power flow prediction.}
We consider AC optimal power flow prediction on the IEEE~30-bus system following~\cite{owerko2025pointwise}. The primal model minimizes expected generation cost while enforcing power-flow equations and operational limits for each demand realization. Both the primal predictor and dual multiplier are parameterized by graph neural networks over the power-network topology.

\subsection{Primal recovery for the parametric dual}

\begin{figure}[t]
    \centering
    \includegraphics[width=\linewidth]{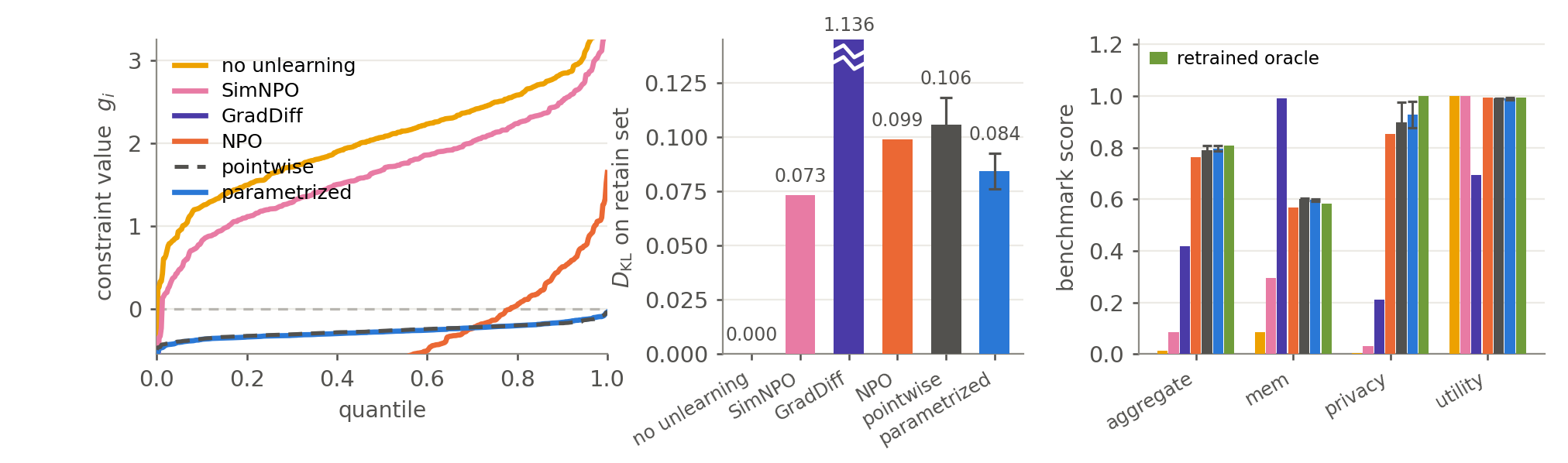}
    \caption{\textbf{Primal recovery on TOFU.}
    \emph{Left:} distribution of forgetting-constraint values, where negative values indicate constraint satisfaction. 
    \emph{Center:} KL divergence on the retain set, measuring deviation from the original model.
    \emph{Right:} aggregate unlearning performance across memorization, privacy, and utility metrics.
    \textit{The parameterized dual closely matches the pointwise in constraint satisfaciton, objective value and downstream performance.}}
    \label{fig:unlearning_primal}
\end{figure}
\paragraph{Dual parameterizations preserve primal performance.}
We first evaluate whether restricting the dual variables to a parametric family degrades the recovered primal solution. Figure~\ref{fig:unlearning_primal} reports results on the TOFU task. Both the pointwise and parameterized dual formulations satisfy the forgetting constraints across essentially the entire forget set, unlike the standard unlearning baselines, which exhibit substantial violations. At the same time, the parameterized formulation preserves the retain distribution particularly well, achieving a retain-set KL divergence of $0.084$, comparable to the pointwise formulation ($0.106$) and lower than the remaining unlearning baselines. This behavior is also reflected in the benchmark metrics, where the parameterized and pointwise formulations achieve nearly identical memorization, privacy, and utility scores. These results show that learning the dual variables through a shared parametric function preserves the primal performance of independent pointwise multipliers while enabling the multiplier to generalize to unseen samples. We observe the same behavior in the other dataset in the additional results provided in Appendix~\ref{app:additional_primal}.

\paragraph{Primal performance remains stable as the number of constraints grows.}
We next evaluate the parameterized dual formulation on the IEEE 30-bus benchmark, a standard power-systems test case with 30 buses, 6 generators, and 41 transmission branches, yielding hundreds of constraints per sample. This setting is particularly challenging for a shared dual parameterization, since different constraint types act on distinct parts of the network and their associated multipliers depend on the complete operating point. Nevertheless, Table~\ref{tab:ieee30-param} shows that the resulting primal GNN remains competitive with standard empirical dual formulations. The parameterized model achieves a negative mean optimality gap of $-3.63\%$, while maintaining constraint violations of the same overall scale as the pointwise-dual baselines. These results suggest that a single learned multiplier function can accommodate hundreds of heterogeneous pointwise constraints without substantially degrading the quality of the recovered primal solution.

\begin{table}[]
\centering
\begin{tabular}{llcccccccc}
\toprule
& & \multicolumn{2}{c}{Optimality Gap (\%)} & \multicolumn{2}{c}{Mean Violation (\%)} & \multicolumn{4}{c}{Max Violation (\%)} \\
\cmidrule(lr){3-4} \cmidrule(lr){5-6} \cmidrule(lr){7-10}
& & Mean & Std & Mean & Std & Mean & Std & P95 & Max \\
\midrule
& Param     & $-3.63$ & 1.15 & 2.70 & 1.26 & 6.26 & 2.61 & 8.04  & 15.80 \\
& Dual-P    & 0.18    & 1.41 & 0.49 & 0.62 & 1.83 & 2.15 & 5.65  & 19.94 \\
& Dual-H    & $-0.08$ & 0.55 & 1.05 & 0.30 & 5.61 & 0.27 & 6.03  & 6.65  \\
& Dual-S    & 2.18    & 1.25 & 1.34 & 0.30 & 8.22 & 1.98 & 11.61 & 13.47 \\
\bottomrule
\end{tabular}
\caption{\textbf{Primal recovery on the IEEE 30-bus benchmark.}
Param replaces the independent per-sample multipliers of Dual-P with a learned dual network conditioned on primal embeddings. Metrics are computed per test operating point and aggregated across the test set. The remaining baselines are reproduced from~\cite{owerko2025pointwise}, Fig.~10.}
\label{tab:ieee30-param}
\end{table}

\subsection{Learnability of the dual multipliers}
\begin{table}[b]
    \centering
    \setlength{\tabcolsep}{7pt}
    \renewcommand{\arraystretch}{1.15}
    \begin{tabular}{lcccc}
        \toprule
        & \multirow{2}{*}{\shortstack{Active-set\\base rate}}
        & \multicolumn{3}{c}{Dual recovery on unseen samples} \\
        \cmidrule(lr){3-5}
        Dataset
        & 
        & Tight-set AUC $\uparrow$
        & NDCG $\uparrow$
        & Spearman $\rho$ $\uparrow$ \\
        \midrule
        MNIST
        & $0.47\%$
        & $0.923 \pm 0.013$
        & $0.836 \pm 0.036$
        & $0.277 \pm 0.026$ \\

        OPF
        & $5.13\%$
        & $0.877 \pm 0.063$
        & $0.902 \pm 0.008$
        & $0.324 \pm 0.035$ \\

        TOFU
        & $80.00\%$
        & $0.713 \pm 0.078$
        & $0.386 \pm 0.040$
        & $-0.048 \pm 0.051$ \\

        FLORA-Bench
        & $54.08\%$
        & ${0.935 \pm 0.001}$
        & ${0.914 \pm 0.001}$
        & $0.211 \pm 0.008$ \\
        \bottomrule
    \end{tabular}
    \caption{\textbf{Generalization of learned dual multipliers to unseen samples.}
    Predicted multipliers are evaluated against an independently computed pointwise dual reference. Tight-set AUC measures identification of active constraints, while NDCG and Spearman's $\rho$ measure recovery of the multiplier ranking. Results are reported as mean $\pm$ standard deviation.}
    \label{tab:dual_generalization}
\end{table}

\paragraph{Learned multipliers generalize to unseen samples.}
We next evaluate whether the learned multiplier preserves dual information on unseen samples. Table~\ref{tab:dual_generalization} compares predicted multipliers with an independently computed pointwise reference solution. Across MNIST, OPF, and FLORA-Bench, the learned multiplier reliably identifies active constraints, with tight-set AUCs between $0.877$ and $0.935$, and ranks the most important constraints well, with NDCG between $0.836$ and $0.914$. Global Spearman correlations are lower, as pointwise multipliers are highly sparse and concentrated near zero for inactive constraints, making their relative ordering poorly conditioned. TOFU presents a different regime, with about $80\%$ of constraints active, leaving fewer inactive samples for separation and less variation across multiplier values. Ranking performance therefore degrades, although active-set identification remains nontrivial, with AUC $0.713$. Overall, the parametric dual transfers meaningful constraint-importance information to unseen samples, particularly through active-set identification and top-ranked sensitivity.

\paragraph{Multiplier magnitude captures sample-level sensitivity.}
We test whether the learned multiplier predicts how strongly an unseen sample constrains the optimization problem. On MNIST, we form held-out sets with different average predicted multipliers and re-solve the problem after imposing their pointwise constraints. Figure~\ref{fig:mnist_sensitivity} shows that the objective increase grows monotonically with the mean predicted multiplier, while random sets produce a nearly constant baseline. This agrees with the sensitivity interpretation in Section~\ref{sec:sensitivity}: larger $\lambda_\phi$ identifies constraints whose enforcement has a greater effect on the optimum. The examples on the right illustrate this behavior: an unambiguous zero receives uniformly small multipliers, whereas an ambiguous seven resembling a one receives a large multiplier for the corresponding class constraint. Such samples require greater classification capacity, explaining their larger effect on the objective. Additional FLORA-Bench results are provided in Appendix~\ref{app:additional_sensitivity}.

\begin{figure}
    \centering
    \includegraphics[width=0.8\linewidth]{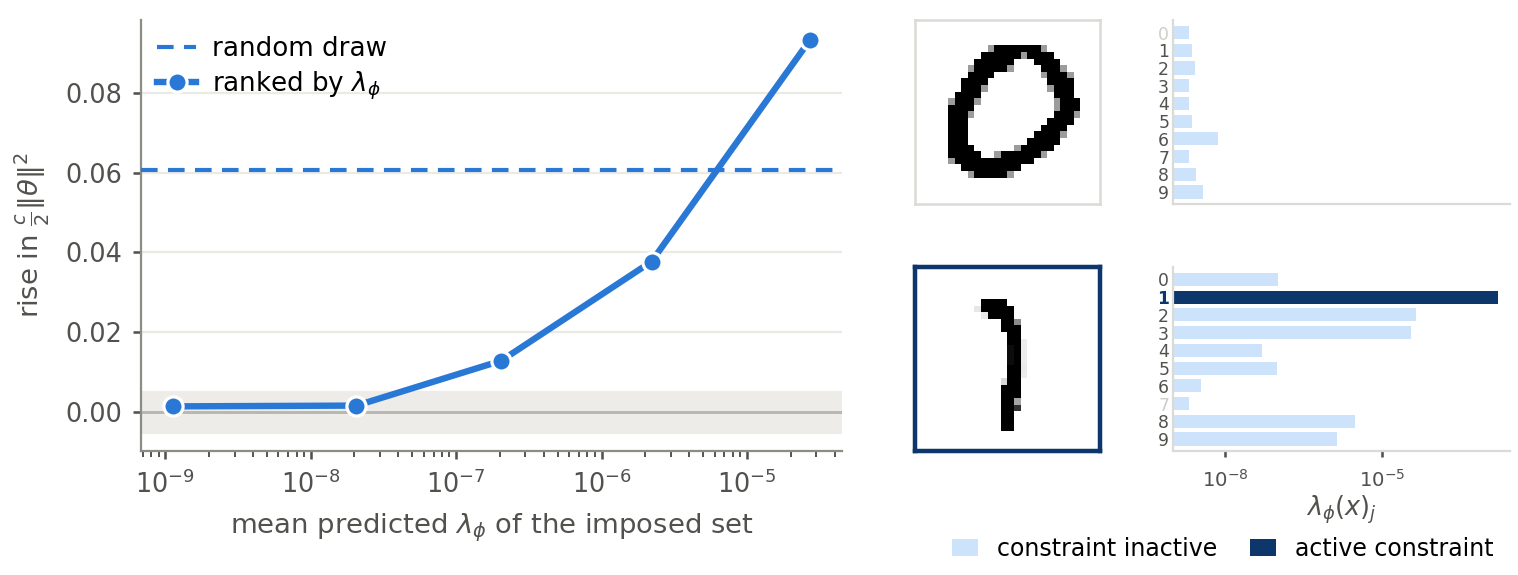}
    \caption{\textbf{Sensitivity estimation on MNIST.}
    \emph{Left:} objective increase after imposing held-out constraints versus their mean predicted multiplier. Larger $\lambda_\phi$ values induce larger increases than random draws. \emph{Right:} representative low- and high-sensitivity samples with their classwise multipliers.}
    \label{fig:mnist_sensitivity}
\end{figure}

\paragraph{Dual capacity balances approximation and generalization.}
We study how dual-model capacity affects sensitivity recovery. Figure~\ref{fig:dual_capacity} shows that increasing the number of dual parameters improves ranking recovery on both training and test sets, indicating reduced approximation error, but the gains saturate beyond roughly $10^7$ parameters. The right panel makes the tradeoff explicit: approximation error decreases sharply, while generalization error initially grows and then levels off. This empirically matches Theorem~\ref{thm:empirical-sensitivity}: richer dual classes better approximate the optimal multiplier but incur greater statistical complexity. Additional MNIST capacity results are provided in Appendix~\ref{app:additional_capacity}.

\begin{figure}
    \centering
    \includegraphics[width=0.8\linewidth]{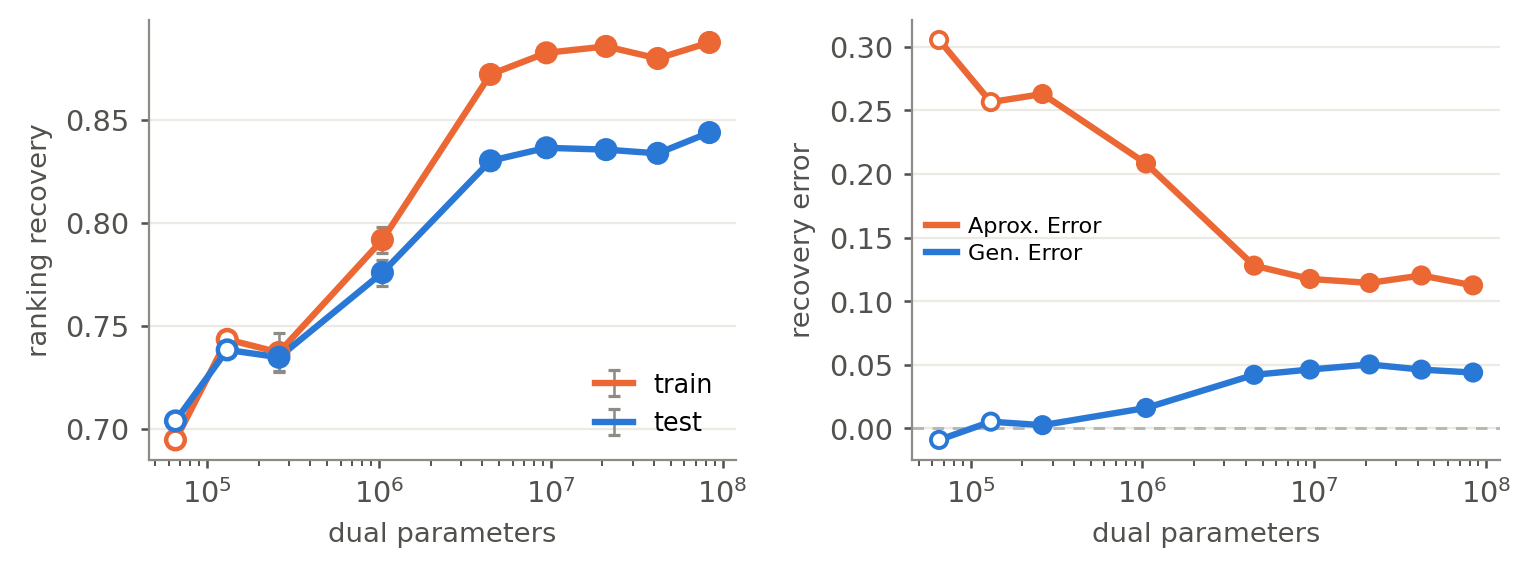}
    \caption{\textbf{Effect of dual parameterization capacity on FLORA-Bench.}
    \emph{Left:} ranking recovery on the training and test sets as the number of dual-network parameters increases.
    \emph{Right:} decomposition into approximation and generalization errors.
    Increasing capacity substantially reduces approximation error and improves sensitivity recovery, while introducing a growing generalization error, illustrating the tradeoff between dual expressivity and statistical complexity.}
    \label{fig:dual_capacity}
\end{figure}

\newpage

\subsection*{AI use statement}

In this work, we used generative AI tools to assist with the implementation of experimental methods, the development and editing of software code, the selection and refinement of experimental parameters, and aspects of experimental methodology. We also used generative AI tools to assist in the verification of mathematical arguments and proofs. AI-assisted mathematical content was independently reviewed and verified by the authors.

We additionally used generative AI tools to assist with improving the clarity, grammar and readability organization of the manuscript, and identifying relevant literature and references. All AI-assisted code was manually reviewed and tested for correctness, and all AI-assisted mathematical claims, proofs, experimental conclusions, and citations were checked by the authors.

We did not use generative AI tools to generate synthetic datasets, formulate the central theoretical models or conceptual framework of the work, formulate the main mathematical claims, propose or refine the primary research hypotheses, translate research materials, clean or reformat datasets, or perform qualitative or thematic data analysis.

The authors have reviewed all AI-assisted work and take responsibility for the final content of this work, including all text, mathematical claims, proofs, code, experimental results, citations, and other artifacts produced with the aid of generative AI.

\subsection*{Ethics statement}

Given the nature of the manuscript, we do not believe an additional ethic statement is needed.

\subsection*{Reproducibility statement}

To facilitate reproducibility, we provide one ready-to-run repository for each task, enabling reproduction of the experiments and results reported throughout the manuscript. The repositories are designed to be as simple and self-contained as possible and include comments to improve clarity and ease of use.
%A complete description of the hardware and computational environment used to obtain the experimental results is also provided in the Appendix.

Regarding the theoretical results, complete proofs of all theoretical statements are provided in the Appendix, with additional intermediate steps and explanations to facilitate their verification.

\bibliography{refs}
\bibliographystyle{iclr2027_conference}

\appendix
\label{sec:apx}
\newpage

\section*{Appendix}
The appendix is organized as follows:
\begin{itemize}
    \item Theoretical Results (Appendix~\ref{apx:a_theo}): Proofs of the theorems on the main body, derivation of uniform convergance rates from Radamacher complexities and extension to Augmented Lagrangian.
    
    \item Task Details (Appendix~\ref{apx:c_exper_det}): Description of the different tasks, baselines, datasets and relevant hyperparameters used in our experiments.

    \item Implementation Details (Appendix~\ref{apx:b_algo}): Algorithm implementation and primal/dual updates, evaluation of primal recover performance and dual learning.

    \item Additional Results (Appendix~\ref{apx:d_add_res}): Additional figures and tables.    
\end{itemize}

\section{Theoretical Results}
\label{apx:a_theo}
\subsection{Derivation of Uniform Convergence Rates from Rademacher Complexity}
\label{apx:rademacher}

In this section, we show that the uniform-convergence rates assumed in
Assumptions~1 and~7 follow from standard Rademacher-complexity arguments.
We further record the consequences for the weighted and positive-part loss
classes used in the proofs.

Let
\begin{align}
    \mathcal{L}_0
    &:=
    \left\{
        z \mapsto \ell_0(f_\theta,z)
        \;:\;
        f_\theta\in\mathcal{H}
    \right\}, \\
    \mathcal{L}
    &:=
    \left\{
        z \mapsto \ell(f_\theta,z)
        \;:\;
        f_\theta\in\mathcal{H}
    \right\},
\end{align}

denote the loss classes induced by the hypothesis class. 

For a function class $\mathcal{G}$ and an i.i.d. sample
$S=\{z_n\}_{n=1}^N$, define its Rademacher complexity as

\begin{equation}
    \mathfrak{R}_N(\mathcal{G})
    :=
    \mathbb{E}_{S,\sigma}
    \left[
        \sup_{g\in\mathcal{G}}
        \left|
            \frac{1}{N}
            \sum_{n=1}^N \sigma_n g(z_n)
        \right|
    \right],
\end{equation}

where $\sigma_1,\dots,\sigma_N$ are independent Rademacher random variables.
Standard symmetrization and concentration arguments imply that, for a class
satisfying $\|g\|_\infty\le M$, with probability at least $1-\delta$,
\begin{equation}
\label{eq:rad_uniform}
    \sup_{g\in\mathcal{G}}
    \left|
        \mathbb{E}[g(z)]
        -
        \frac{1}{N}\sum_{n=1}^N g(z_n)
    \right|
    \le
    2\mathfrak{R}_N(\mathcal{G})
    +
    M\sqrt{\frac{2\log(2/\delta)}{N}}.
\end{equation}

Consequently, the following rates satisfy Assumption~1:
\begin{align}
\label{eq:zeta0_rad}
    \zeta_0(N_0,\delta)
    &:=
    4\mathfrak{R}_{N_0}(\mathcal{L}_0)
    +
    2B_0
    \sqrt{\frac{2\log(2/\delta)}{N_0}}, \\
\label{eq:zeta_rad}
    \zeta(N,\delta)
    &:=
    4\mathfrak{R}_{N}(\mathcal{L})
    +
    2B
    \sqrt{\frac{2\log(2/\delta)}{N}}.
\end{align}
In particular, if the corresponding Rademacher complexities vanish with the sample size, as is the case for standard neural networks~\citep{golowich2019sizeindependentsamplecomplexityneural},  then $\zeta_0(N_0,\delta),\zeta(N,\delta)\to0$. 

Rademacher-complexity bounds can be transferred directly to several loss classes that arise in our analysis. In particular, Lipschitz transformations and bounded reweightings preserve uniform convergence, up to the corresponding Lipschitz or magnitude factors. We formalize these consequences using the Rademacher contraction inequality.

\begin{lemma}[Rademacher contraction inequality]
\label{lemma:radam_contraction}
Let $\mathcal F$ be a class of real-valued functions and let
$\psi:\mathbb R\to\mathbb R$ be $L$-Lipschitz with $\psi(0)=0$. Then, for any sample
$S=(x_1,\ldots,x_N)$,
\[
    \widehat{\mathfrak R}_S(\psi\circ\mathcal F)
    \leq
    L\,\widehat{\mathfrak R}_S(\mathcal F).
\]
Consequently,
\[
    \mathfrak R_N(\psi\circ\mathcal F)
    \leq
    L\,\mathfrak R_N(\mathcal F).
\]
\end{lemma}
\begin{proof}
See~\cite[Theorem 4.12]{alma990081500390107871}.
\end{proof}

The following proposition collects the resulting uniform-convergence bounds used throughout the proof of the theorems.

\begin{proposition}[Uniform convergence for derived loss classes]
\label{prop:derived_uc}
Let Assumption~\ref{as:uniform-convergence} hold, and let the uniform convergence rate be derived from \ref{eq:zeta_rad}. Then the following hold.

\begin{enumerate}
    \item \textbf{Positive-part loss.}
    Let
    \[
        \mathcal L_+
        :=
        \left\{
            z\mapsto[\ell(f_\theta,z)]_+
            :
            f_\theta\in\mathcal H
        \right\}.
    \]
    Then
    \[
        \mathfrak R_N(\mathcal L_+)
        \leq
        \mathfrak R_N(\mathcal L),
    \]
    and, with probability at least $1-\delta$,
    \begin{equation}
    \label{eq:positive_part_uc}
        \sup_{f_\theta\in\mathcal H}
        \left|
            \mathbb E_{\mathcal D}
            \big[[\ell(f_\theta,z)]_+\big]
            -
            \frac{1}{N}
            \sum_{z_n\in S}
            [\ell(f_\theta,z_n)]_+
        \right|
        \leq
        \zeta(N,\delta).
    \end{equation}

    \item \textbf{Fixed weighted loss.}
    Let $\lambda:\mathcal Z\to\mathbb R_+$ be fixed independently of the
    sample and satisfy $\|\lambda\|_\infty\leq\Gamma$. Define
    \[
        \lambda\mathcal L
        :=
        \left\{
            z\mapsto\lambda(z)\ell(f_\theta,z)
            :
            f_\theta\in\mathcal H
        \right\}.
    \]
    Then
    \[
        \mathfrak R_N(\lambda\mathcal L)
        \leq
        \Gamma\,\mathfrak R_N(\mathcal L),
    \]
    and, with probability at least $1-\delta$,
    \begin{equation}
    \label{eq:weighted_uc}
        \sup_{f_\theta\in\mathcal H}
        \left|
            \mathbb E_{\mathcal D}
            [\lambda(z)\ell(f_\theta,z)]
            -
            \frac{1}{N}
            \sum_{z_n\in S}
            \lambda(z_n)\ell(f_\theta,z_n)
        \right|
        \leq
        \Gamma\,\zeta(N,\delta).
    \end{equation}

    \item \textbf{Learned dual class.}
    Suppose $\|\lambda_\phi\|_\infty\leq\Gamma$ for every
    $\lambda_\phi\in\Lambda_\Phi$ and let Assumption~\ref{as:dual-uc} hold with a rate of the form
    \begin{equation}
    \label{eq:zeta_lambda_rad}
        \zeta_\Lambda(N,\delta)
        :=
        4\mathfrak R_N(\Lambda_\Phi)
        +
        2\Gamma
        \sqrt{\frac{2\log(2/\delta)}{N}}.
    \end{equation}
    Then, for the product class
    \[
        \Lambda_\Phi\mathcal L
        :=
        \left\{
            z\mapsto
            \lambda_\phi(z)\ell(f_\theta,z)
            :
            \lambda_\phi\in\Lambda_\Phi,\;
            f_\theta\in\mathcal H
        \right\},
    \]
    it holds that
    \[
        \mathfrak R_N(\Lambda_\Phi\mathcal L)
        \leq
        2\left[
            \Gamma\mathfrak R_N(\mathcal L)
            +
            B\mathfrak R_N(\Lambda_\Phi)
        \right],
    \]
    and thus, with probability at least $1-\delta$,
    \begin{align}
    \label{eq:product_uc}
        \sup_{\substack{
            \lambda_\phi\in\Lambda_\Phi\\
            f_\theta\in\mathcal H}}
        \left|
            \mathbb E_{\mathcal D}
            [\lambda_\phi(z)\ell(f_\theta,z)]
            -
            \frac{1}{N}
            \sum_{z_n\in S}
            \lambda_\phi(z_n)\ell(f_\theta,z_n)
        \right|
        \leq
        \Gamma\zeta(N,\delta)
        +
        B\zeta_\Lambda(N,\delta).
    \end{align}
\end{enumerate}
\end{proposition}

\begin{proof}
For the positive-part loss, the map $t\mapsto[t]_+$ is $1$-Lipschitz, so the Rademacher contraction inequality in \ref{lemma:radam_contraction} gives
\[
    \mathfrak R_N(\mathcal L_+)
    \leq
    \mathfrak R_N(\mathcal L).
\]
Since $[\ell(f_\theta,z)]_+\leq B$, applying the uniform-convergence bound in
\eqref{eq:rad_uniform} yields~\eqref{eq:positive_part_uc}.

For the fixed weighted loss, $t\mapsto\lambda(z_n)t$ is $\Gamma$-Lipschitz. Applying the contraction inequality in \ref{lemma:radam_contraction} therefore gives
\[
    \mathfrak R_N(\lambda\mathcal L)
    \leq
    \Gamma\mathfrak R_N(\mathcal L).
\]
Together with
$|\lambda(z)\ell(f_\theta,z)|\leq\Gamma B$,
\eqref{eq:rad_uniform} gives~\eqref{eq:weighted_uc}.

Finally,  the jointly learned multiplier and loss,
the standard Rademacher product-class inequality gives
\[
    \mathfrak R_N(\Lambda_\Phi\mathcal L)
    \leq
    2\left[
        \Gamma\mathfrak R_N(\mathcal L)
        +
        B\mathfrak R_N(\Lambda_\Phi)
    \right].
\]
Since
$|\lambda_\phi(z)\ell(f_\theta,z)|\leq\Gamma B$, applying
\eqref{eq:rad_uniform} and substituting the definitions of
$\zeta(N,\delta)$ and $\zeta_\Lambda(N,\delta)$ yields
\eqref{eq:product_uc}.
\end{proof}

\subsection{Proof of Lemma 1}
\label{apx:dual-learn}
\begin{proof}
Let the empirical dual function be
\begin{equation}
    \widehat g(\lambda)
    :=
    \inf_{f_\theta\in\mathcal H}
    \left\{
        \frac{1}{N_0}\sum_{n=1}^{N_0}
        \ell_0(f_\theta,z_n)
        +
        \frac{1}{N}\sum_{n=1}^{N}
        \lambda(z_n)\ell(f_\theta,z_n)
    \right\}.
\end{equation}

We first bound $D^\star-\widehat D_\Phi^\star$. For any $\epsilon>0$, by Assumption~\ref{as:dual-approximation}, there exists a fixed $\bar\lambda_\phi\in\Lambda_\Phi$ such that
\begin{equation}
    \|\bar\lambda_\phi-\lambda^\star\|_{L_1(D)}
    \leq
    \nu_\Phi+\epsilon.
\end{equation}
Since $|\ell(f_\theta,z)|\leq B$, for every $f_\theta\in\mathcal H$,
\begin{equation}
    \left|
        \E_D[
            \bar\lambda_\phi(z)\ell(f_\theta,z)
        ]
        -
        \E_D[
            \lambda^\star(z)\ell(f_\theta,z)
        ]
    \right|
    \leq
    B(\nu_\Phi+\epsilon).
\end{equation}
Taking the infimum over $f_\theta$ gives
\begin{equation}
    g(\bar\lambda_\phi)
    \geq
    D^\star-B(\nu_\Phi+\epsilon).
\end{equation}

Because $\bar\lambda_\phi$ is fixed independently of the sample and satisfies $\|\bar\lambda_\phi\|_\infty\leq\Gamma$, the weighted uniform-convergence result of Proposition~\ref{prop:derived_uc} yields
\begin{equation}
    \widehat g(\bar\lambda_\phi)
    \geq
    g(\bar\lambda_\phi)
    -
    \zeta_0(N_0,\delta)
    -
    \Gamma\zeta(N,\delta).
\end{equation}
Since $\widehat D_\Phi^\star$ maximizes the empirical dual over $\Lambda_\Phi$,
\begin{align}
    \widehat D_\Phi^\star
    &\geq
    \widehat g(\bar\lambda_\phi)
    \\
    &\geq
    D^\star
    -
    B(\nu_\Phi+\epsilon)
    -
    \zeta_0(N_0,\delta)
    -
    \Gamma\zeta(N,\delta).
\end{align}
Letting $\epsilon\to0$ gives
\begin{equation}
    D^\star-\widehat D_\Phi^\star
    \leq
    B\nu_\Phi
    +
    \zeta_0(N_0,\delta)
    +
    \Gamma\zeta(N,\delta).
\label{eq:dual-lemma-upper}
\end{equation}

For the opposite direction, let $\widehat D^\star$ denote the unrestricted
empirical dual optimum obtained by assigning an independent nonnegative
multiplier to each constraint sample. Since the evaluations of every
$\lambda_\phi\in\Lambda_\Phi$ form a feasible multiplier vector for this
unrestricted problem,
\begin{equation}
    \widehat D_\Phi^\star
    \leq
    \widehat D^\star.
\end{equation}

We next relate the unrestricted empirical dual to the statistical dual.
Fix an empirical multiplier vector
$\boldsymbol{\lambda}=(\lambda_1,\ldots,\lambda_N)\in\mathbb R_+^N$
and $\varepsilon>0$. Under mild regularity of the constraint sample
space,\footnote{Specifically, it suffices that the constraint losses vary
uniformly continuously with the sample and that $D$ assigns positive mass
to sufficiently small neighborhoods of points in its support. For example,
this holds when $\mathcal Z$ is a compact metric space and
$z\mapsto\ell(f_\theta,z)$ is uniformly Lipschitz over
$f_\theta\in\mathcal H$. We omit these topological regularity conditions from the main statement to keep the presentation focused on the statistical and approximation aspects of the result.}
one can construct $\lambda_\varepsilon\in L_1^+(D)$ such that
\begin{equation}
    \sup_{f_\theta\in\mathcal H}
    \left|
        \E_D[
            \lambda_\varepsilon(z)\ell(f_\theta,z)
        ]
        -
        \frac{1}{N}
        \sum_{n=1}^N
        \lambda_n\ell(f_\theta,z_n)
    \right|
    \leq
    \varepsilon.
\label{eq:dual-lifting}
\end{equation}
Indeed, this follows by placing the mass $\lambda_n/N$ on sufficiently
small neighborhoods of each $z_n$.

Let $\widehat g(\boldsymbol{\lambda})$ denote the empirical dual value
associated with $\boldsymbol{\lambda}$. On the uniform-convergence event
for the objective,
\begin{align}
    \widehat g(\boldsymbol{\lambda})
    &=
    \inf_{f_\theta\in\mathcal H}
    \left\{
        \frac{1}{N_0}
        \sum_{z_n\in S_0}
        \ell_0(f_\theta,z_n)
        +
        \frac{1}{N}
        \sum_{n=1}^N
        \lambda_n\ell(f_\theta,z_n)
    \right\}
    \\
    &\leq
    \inf_{f_\theta\in\mathcal H}
    \left\{
        \E_{D_0}[\ell_0(f_\theta,z)]
        +
        \E_D[
            \lambda_\varepsilon(z)\ell(f_\theta,z)
        ]
    \right\}
    +
    \zeta_0(N_0,\delta)
    +
    \varepsilon
    \\
    &=
    g(\lambda_\varepsilon)
    +
    \zeta_0(N_0,\delta)
    +
    \varepsilon
    \\
    &\leq
    D^\star
    +
    \zeta_0(N_0,\delta)
    +
    \varepsilon.
\end{align}
Since this holds for every empirical multiplier vector and every
$\varepsilon>0$, taking the supremum over
$\boldsymbol{\lambda}\in\mathbb R_+^N$ and then letting
$\varepsilon\to0$ yields
\begin{equation}
    \widehat D^\star
    \leq
    D^\star+\zeta_0(N_0,\delta).
\end{equation}
Consequently,
\begin{equation}
    \widehat D_\Phi^\star-D^\star
    \leq
    \zeta_0(N_0,\delta).
\label{eq:dual-lemma-lower}
\end{equation}

Combining~\eqref{eq:dual-lemma-upper}
and~\eqref{eq:dual-lemma-lower}, and applying a union bound over the
corresponding uniform-convergence events, yields, with probability at
least $1-3\delta$,
\begin{equation}
    \left|
        D^\star-\widehat D_\Phi^\star
    \right|
    \leq
    B\nu_\Phi
    +
    \zeta_0(N_0,\delta)
    +
    \Gamma\zeta(N,\delta)
    =
    \Delta_D.
\end{equation}
\end{proof}
\subsection{Proof of Theorem~\ref{thm:primal-recovery}}
\label{apx:primal-recovery}

\begin{proof} We start by defining notation.

Lets define the convexified empirical primal problem as
\begin{equation}
    \widehat P_{\mathcal C}^\star
    :=
    \min_{f\in\overline{\mathcal H}}
    \frac{1}{N_0}
    \sum_{z_n\in S_0}
    \ell_0(f,z_n)
    \qquad
    \text{s.t.}
    \qquad
    \ell(f,z_n)\leq0,
    \quad z_n\in S.
\end{equation}
For $u,v\in\mathbb R^N$, write
\begin{equation}
    \langle u,v\rangle_N
    :=
    \frac1N\sum_{n=1}^N u_nv_n,
    \qquad
    \|u\|_{2,N}^2
    :=
    \frac1N\sum_{n=1}^N u_n^2,
\end{equation}
and let
\begin{equation}
    \ell_S(f)
    :=
    \bigl(
        \ell(f,z_1),\ldots,\ell(f,z_N)
    \bigr).
\end{equation}
The convexified empirical dual function is
\begin{equation}
    \widehat g_{\mathcal C}(\lambda)
    :=
    \min_{f\in\overline{\mathcal H}}
    \left\{
        \frac{1}{N_0}
        \sum_{z_n\in S_0}
        \ell_0(f,z_n)
        +
        \langle\lambda,\ell_S(f)\rangle_N
    \right\},
    \qquad
    \lambda\in\mathbb R_+^N.
\end{equation}
Let
\begin{equation}
    \widehat D_{\mathcal C}^\star
    :=
    \max_{\lambda\geq0}
    \widehat g_{\mathcal C}(\lambda),
\end{equation}
and denote by
\begin{equation}
    \widehat f_{\mathcal C}(\lambda)
    :=
    \arg\min_{f\in\overline{\mathcal H}}
    \left\{
        \frac{1}{N_0}
        \sum_{z_n\in S_0}
        \ell_0(f,z_n)
        +
        \langle\lambda,\ell_S(f)\rangle_N
    \right\}
\end{equation}
the corresponding Lagrangian minimizer. By strong convexity of the objective, this minimizer is unique. The strict Slater condition implies strong duality for the convexified empirical problem, so if $\lambda_{\mathcal C}^\star$ is an optimal multiplier, then
\begin{equation}
    \widehat f_{\mathcal C}^\star
    :=
    \widehat f_{\mathcal C}(\lambda_{\mathcal C}^\star)
\end{equation}
is feasible for the convexified empirical primal problem.

Finally, let
\begin{equation}
    \widehat\lambda_\Phi^\star
    :=
    \bigl(
        \widehat\lambda_\Phi^\star(z_1),
        \ldots,
        \widehat\lambda_\Phi^\star(z_N)
    \bigr)
\end{equation}
denote the vector obtained by evaluating the learned parametric multiplier on the constraint sample.

\paragraph{Feasibility.}
We decompose the discrepancy between the constraint losses of the returned predictor and those of the optimal convexified solution as
\begin{align}
\left\|
    \ell_S(\widehat f_{\theta_\Phi})
    -
    \ell_S(\widehat f_{\mathcal C}^\star)
\right\|_{2,N}
\leq\;&
\left\|
    \ell_S(
        \widehat f_{\mathcal C}
        (\widehat\lambda_\Phi^\star)
    )
    -
    \ell_S(\widehat f_{\mathcal C}^\star)
\right\|_{2,N}
\nonumber\\
&+
\left\|
    \ell_S(\widehat f_{\theta_\Phi})
    -
    \ell_S(
        \widehat f_{\mathcal C}
        (\widehat\lambda_\Phi^\star)
    )
\right\|_{2,N}.
\label{eq:feasibility-decomposition-proof}
\end{align}

\emph{1) Dual suboptimality and smoothness.}
Strong convexity of the empirical objective and Lipschitz continuity of the constraint losses imply that $\widehat g_{\mathcal C}$ is differentiable and
\begin{equation}
    \nabla
    \widehat g_{\mathcal C}(\lambda)
    =
    \ell_S(
        \widehat f_{\mathcal C}(\lambda)
    ).
\end{equation}
Moreover, its gradient is $L^2/\mu$-Lipschitz. Indeed, for any $\lambda,\lambda'\geq0$, strong convexity of the corresponding Lagrangians and Lipschitz continuity of the constraint map yield
\begin{equation}
    \mu
    \left\|
        \widehat f_{\mathcal C}(\lambda)
        -
        \widehat f_{\mathcal C}(\lambda')
    \right\|^2
    \leq
    \left\langle
        \lambda-\lambda',
        \ell_S(
            \widehat f_{\mathcal C}(\lambda')
        )
        -
        \ell_S(
            \widehat f_{\mathcal C}(\lambda)
        )
    \right\rangle_N,
\end{equation}
and therefore
\begin{equation}
    \left\|
        \nabla\widehat g_{\mathcal C}(\lambda)
        -
        \nabla\widehat g_{\mathcal C}(\lambda')
    \right\|_{2,N}
    \leq
    \frac{L^2}{\mu}
    \|\lambda-\lambda'\|_{2,N}.
\end{equation}

We next bound the dual suboptimality of $\widehat\lambda_\Phi^\star$ for the convexified empirical dual. Since the convexified primal domain contains $\mathcal H$,
\begin{equation}
    \widehat g_{\mathcal C}(\lambda)
    \leq
    \widehat g(\lambda)
\end{equation}
for every $\lambda\geq0$. Furthermore, the unrestricted empirical dual dominates the convexified dual, and the lifting argument used in Lemma~\ref{lem:dual-value} gives
\begin{equation}
    \widehat D_{\mathcal C}^\star
    \leq
    \widehat D^\star
    \leq
    D^\star+\zeta_0.
\end{equation}
Lemma~\ref{lem:dual-value} gives
\begin{equation}
    \widehat D_\Phi^\star
    \geq
    D^\star-\Delta_D.
\end{equation}
Hence,
\begin{equation}
    \widehat D_{\mathcal C}^\star
    -
    \widehat D_\Phi^\star
    \leq
    \Delta_D+\zeta_0.
\label{eq:convex-dual-opt-gap}
\end{equation}

It remains to account for the difference between minimizing the Lagrangian over $\mathcal H$ and over $\overline{\mathcal H}$. Let $f\in\overline{\mathcal H}$. By Assumption~\ref{as:approx-convexity}, there exists $f_\theta\in\mathcal H$ satisfying $\|f_\theta-f\|_\infty\leq\nu_{\mathcal H}$. Lipschitz continuity then gives
\begin{equation}
    \left|
        \frac1{N_0}
        \sum_{z_n\in S_0}
        \ell_0(f_\theta,z_n)
        -
        \frac1{N_0}
        \sum_{z_n\in S_0}
        \ell_0(f,z_n)
    \right|
    \leq
    L\nu_{\mathcal H},
\end{equation}
and
\begin{equation}
    \left|
        \ell(f_\theta,z_n)
        -
        \ell(f,z_n)
    \right|
    \leq
    L\nu_{\mathcal H}
\end{equation}
for every $z_n\in S$. Since $0\leq\widehat\lambda_\Phi^\star(z_n)\leq\Gamma$,
\begin{equation}
    0
    \leq
    \widehat g(\widehat\lambda_\Phi^\star)
    -
    \widehat g_{\mathcal C}(\widehat\lambda_\Phi^\star)
    \leq
    L(1+\Gamma)\nu_{\mathcal H}.
\label{eq:empirical-convexification-gap}
\end{equation}
Combining \eqref{eq:convex-dual-opt-gap} and \eqref{eq:empirical-convexification-gap},
\begin{equation}
    \widehat D_{\mathcal C}^\star
    -
    \widehat g_{\mathcal C}(\widehat\lambda_\Phi^\star)
    \leq
    \Delta_D
    +
    \zeta_0
    +
    L(1+\Gamma)\nu_{\mathcal H}.
\label{eq:convexified-dual-suboptimality-proof}
\end{equation}

Since $-\widehat g_{\mathcal C}$ is convex and $L^2/\mu$-smooth, cocoercivity gives
\begin{align}
&
\left\|
    \nabla\widehat g_{\mathcal C}
    (\widehat\lambda_\Phi^\star)
    -
    \nabla\widehat g_{\mathcal C}
    (\lambda_{\mathcal C}^\star)
\right\|_{2,N}^2
\nonumber\\
&\qquad\leq
\frac{2L^2}{\mu}
\left[
    \widehat D_{\mathcal C}^\star
    -
    \widehat g_{\mathcal C}
    (\widehat\lambda_\Phi^\star)
\right].
\end{align}
Using $\nabla\widehat g_{\mathcal C}(\lambda)=\ell_S(\widehat f_{\mathcal C}(\lambda))$ and \eqref{eq:convexified-dual-suboptimality-proof}, we obtain
\begin{align}
&
\left\|
    \ell_S(
        \widehat f_{\mathcal C}
        (\widehat\lambda_\Phi^\star)
    )
    -
    \ell_S(
        \widehat f_{\mathcal C}^\star
    )
\right\|_{2,N}
\nonumber\\
&\qquad\leq
\sqrt{
    \frac{2L^2}{\mu}
    \left[
        \Delta_D
        +
        \zeta_0
        +
        L(1+\Gamma)\nu_{\mathcal H}
    \right]
}.
\label{eq:first-feasibility-part}
\end{align}

\emph{2) Transfer to the original hypothesis class.}
By \eqref{eq:empirical-convexification-gap} and the fact that $\widehat f_{\theta_\Phi}$ exactly minimizes the empirical Lagrangian over $\mathcal H$,
\begin{align}
&
\widehat L(
    \widehat f_{\theta_\Phi},
    \widehat\lambda_\Phi^\star
)
-
\widehat g_{\mathcal C}
(\widehat\lambda_\Phi^\star)
\nonumber\\
&\qquad=
\widehat g(
    \widehat\lambda_\Phi^\star
)
-
\widehat g_{\mathcal C}
(
    \widehat\lambda_\Phi^\star
)
\leq
L(1+\Gamma)\nu_{\mathcal H}.
\end{align}
The convexified Lagrangian is $\mu$-strongly convex. Since $\widehat f_{\theta_\Phi}\in\mathcal H\subseteq\overline{\mathcal H}$,
\begin{equation}
    \frac{\mu}{2}
    \left\|
        \widehat f_{\theta_\Phi}
        -
        \widehat f_{\mathcal C}
        (\widehat\lambda_\Phi^\star)
    \right\|^2
    \leq
    L(1+\Gamma)\nu_{\mathcal H}.
\end{equation}
Lipschitz continuity of the constraint loss therefore gives
\begin{align}
&
\left\|
    \ell_S(\widehat f_{\theta_\Phi})
    -
    \ell_S(
        \widehat f_{\mathcal C}
        (\widehat\lambda_\Phi^\star)
    )
\right\|_{2,N}
\nonumber\\
&\qquad\leq
\sqrt{
    \frac{
        2L^3(1+\Gamma)\nu_{\mathcal H}
    }{\mu}
}.
\label{eq:second-feasibility-part}
\end{align}

Substituting \eqref{eq:first-feasibility-part} and \eqref{eq:second-feasibility-part} into \eqref{eq:feasibility-decomposition-proof} gives
\begin{align}
&
\left\|
    \ell_S(\widehat f_{\theta_\Phi})
    -
    \ell_S(\widehat f_{\mathcal C}^\star)
\right\|_{2,N}
\nonumber\\
&\qquad\leq
\sqrt{
    \frac{2L^2}{\mu}
    \left[
        \Delta_D
        +
        \zeta_0
        +
        L(1+\Gamma)\nu_{\mathcal H}
    \right]
}
+
\sqrt{
    \frac{
        2L^3(1+\Gamma)\nu_{\mathcal H}
    }{\mu}
}.
\label{eq:empirical-constraint-discrepancy}
\end{align}
Since $\widehat f_{\mathcal C}^\star$ is feasible,
\begin{equation}
    \ell(
        \widehat f_{\mathcal C}^\star,
        z_n
    )
    \leq0
    \qquad
    \text{for every }z_n\in S.
\end{equation}
Therefore,
\begin{equation}
    [\ell(\widehat f_{\theta_\Phi},z_n)]_+
    \leq
    \left|
        \ell(\widehat f_{\theta_\Phi},z_n)
        -
        \ell(\widehat f_{\mathcal C}^\star,z_n)
    \right|,
\end{equation}
and hence, by Cauchy--Schwarz,
\begin{align}
\frac1N
\sum_{z_n\in S}
[\ell(\widehat f_{\theta_\Phi},z_n)]_+
\leq\;&
\sqrt{
    \frac{2L^2}{\mu}
    \left[
        \Delta_D
        +
        \zeta_0
        +
        L(1+\Gamma)\nu_{\mathcal H}
    \right]
}
\nonumber\\
&+
\sqrt{
    \frac{
        2L^3(1+\Gamma)\nu_{\mathcal H}
    }{\mu}
}.
\end{align}
Finally, since $t\mapsto[t]_+$ is $1$-Lipschitz and vanishes at the origin, the contraction result of Proposition~\ref{prop:derived_uc} implies that the positive-part loss has uniform-convergence rate no larger than $\zeta$. Thus,
\begin{align}
\E_D
\left[
    [\ell(\widehat f_{\theta_\Phi},z)]_+
\right]
\leq\;&
\sqrt{
    \frac{2L^2}{\mu}
    \left[
        \Delta_D
        +
        \zeta_0
        +
        L(1+\Gamma)\nu_{\mathcal H}
    \right]
}
\nonumber\\
&+
\sqrt{
    \frac{
        2L^3(1+\Gamma)\nu_{\mathcal H}
    }{\mu}
}
+
\zeta
=
\Delta_F,
\end{align}
which proves \eqref{eq:statistical-feasibility}.

\paragraph{Optimality.}
We first control the statistical primal--dual gap. Let $P_{\mathcal C}^\star$ denote the optimal value of the statistical primal problem over $\overline{\mathcal H}$, and let $P_{\mathcal C,-L\nu_{\mathcal H}}^\star$ denote the same problem with the constraint tightened to
\begin{equation}
    \ell(f,z)
    \leq
    -L\nu_{\mathcal H}
    \qquad
    D\text{-a.e.}
\end{equation}
If $f$ is feasible for the tightened convexified problem, Assumption~\ref{as:approx-convexity} provides $f_\theta\in\mathcal H$ satisfying $\|f_\theta-f\|_\infty\leq\nu_{\mathcal H}$. By Lipschitz continuity, $f_\theta$ is feasible for the original statistical problem and its objective differs by at most $L\nu_{\mathcal H}$. Hence,
\begin{equation}
    P^\star
    \leq
    P_{\mathcal C,-L\nu_{\mathcal H}}^\star
    +
    L\nu_{\mathcal H}.
\label{eq:primal-gap-first}
\end{equation}

By Assumption~\ref{as:slater}, $\bar f_\theta$ satisfies the tightened convexified constraint with margin $s$. If $\lambda^\star_{\mathcal C,-L\nu_{\mathcal H}}$ is an optimal multiplier of the tightened convexified problem, strong duality and evaluation of its Lagrangian at $\bar f_\theta$ give
\begin{equation}
    \|\lambda^\star_{\mathcal C,-L\nu_{\mathcal H}}\|_{L_1(D)}
    \leq
    \frac{B_0}{s}.
\end{equation}
The sensitivity of the convex value function with respect to tightening the constraint therefore yields
\begin{equation}
    P_{\mathcal C,-L\nu_{\mathcal H}}^\star
    -
    P_{\mathcal C}^\star
    \leq
    \frac{B_0L}{s}
    \nu_{\mathcal H}.
\end{equation}
Since the statistical dual over $\mathcal H$ dominates the dual of the convexified problem,
\begin{equation}
    D^\star
    \geq
    P_{\mathcal C}^\star,
\end{equation}
while weak duality gives $D^\star\leq P^\star$. Combining these facts with \eqref{eq:primal-gap-first},
\begin{equation}
    0
    \leq
    P^\star-D^\star
    \leq
    L
    \left(
        1+
        \frac{B_0}{s}
    \right)
    \nu_{\mathcal H}.
\label{eq:statistical-primal-dual-gap-proof}
\end{equation}

We next control the empirical complementary-slackness term associated with $\widehat\lambda_\Phi^\star$. Smoothness of $\widehat g_{\mathcal C}$ implies that its projected gradient mapping satisfies
\begin{equation}
    \left\|
        \frac{L^2}{\mu}
        \left[
            \Pi_{\mathbb R_+^N}
            \left(
                \widehat\lambda_\Phi^\star
                +
                \frac{\mu}{L^2}
                \nabla\widehat g_{\mathcal C}
                (\widehat\lambda_\Phi^\star)
            \right)
            -
            \widehat\lambda_\Phi^\star
        \right]
    \right\|_{2,N}
    \leq
    \sqrt{
        \frac{2L^2}{\mu}
        \left[
            \Delta_D
            +
            \zeta_0
            +
            L(1+\Gamma)\nu_{\mathcal H}
        \right]
    }.
\label{eq:projected-gradient-bound-proof}
\end{equation}
Using $\nabla\widehat g_{\mathcal C}(\widehat\lambda_\Phi^\star)=\ell_S(\widehat f_{\mathcal C}(\widehat\lambda_\Phi^\star))$, $0\leq\widehat\lambda_\Phi^\star(z_n)\leq\Gamma$, and $|\ell|\leq B$, the coordinatewise form of the projection gives
\begin{align}
&
\left\|
    \widehat\lambda_\Phi^\star
    \odot
    \ell_S(
        \widehat f_{\mathcal C}
        (\widehat\lambda_\Phi^\star)
    )
\right\|_{1,N}
\nonumber\\
&\qquad\leq
\max
\left\{
    \Gamma,
    \frac{B\mu}{L^2}
\right\}
\sqrt{
    \frac{2L^2}{\mu}
    \left[
        \Delta_D
        +
        \zeta_0
        +
        L(1+\Gamma)\nu_{\mathcal H}
    \right]
}.
\label{eq:convexified-complementarity-proof}
\end{align}
Transferring from the convexified Lagrangian minimizer to $\widehat f_{\theta_\Phi}$ using \eqref{eq:second-feasibility-part} gives
\begin{align}
&
\left\|
    \widehat\lambda_\Phi^\star
    \odot
    \ell_S(
        \widehat f_{\theta_\Phi}
    )
\right\|_{1,N}
\nonumber\\
&\qquad\leq
\max
\left\{
    \Gamma,
    \frac{B\mu}{L^2}
\right\}
\sqrt{
    \frac{2L^2}{\mu}
    \left[
        \Delta_D
        +
        \zeta_0
        +
        L(1+\Gamma)\nu_{\mathcal H}
    \right]
}
\nonumber\\
&\qquad\quad+
\Gamma
\sqrt{
    \frac{
        2L^3(1+\Gamma)\nu_{\mathcal H}
    }{\mu}
}
\nonumber\\
&\qquad\leq
\left(
    \Gamma
    +
    \frac{B\mu}{L^2}
\right)
\Delta_F.
\label{eq:complementarity-final-proof}
\end{align}
In particular,
\begin{equation}
    \left|
        \left\langle
            \widehat\lambda_\Phi^\star,
            \ell_S(\widehat f_{\theta_\Phi})
        \right\rangle_N
    \right|
    \leq
    \left(
        \Gamma
        +
        \frac{B\mu}{L^2}
    \right)
    \Delta_F.
\label{eq:complementarity-inner-proof}
\end{equation}

Finally, exact empirical Lagrangian minimization gives
\begin{equation}
    \widehat D_\Phi^\star
    =
    \frac1{N_0}
    \sum_{z_n\in S_0}
    \ell_0(
        \widehat f_{\theta_\Phi},
        z_n
    )
    +
    \left\langle
        \widehat\lambda_\Phi^\star,
        \ell_S(
            \widehat f_{\theta_\Phi}
        )
    \right\rangle_N.
\end{equation}
Therefore,
\begin{align}
&
\E_{D_0}
\left[
    \ell_0(
        \widehat f_{\theta_\Phi},z
    )
\right]
-
P^\star
\nonumber\\
&=
\left(
    \E_{D_0}
    [
        \ell_0(
            \widehat f_{\theta_\Phi},z
        )
    ]
    -
    \frac1{N_0}
    \sum_{z_n\in S_0}
    \ell_0(
        \widehat f_{\theta_\Phi},z_n
    )
\right)
\nonumber\\
&\quad+
\left(
    \widehat D_\Phi^\star
    -
    D^\star
\right)
+
\left(
    D^\star
    -
    P^\star
\right)
-
\left\langle
    \widehat\lambda_\Phi^\star,
    \ell_S(
        \widehat f_{\theta_\Phi}
    )
\right\rangle_N.
\end{align}
Applying Assumption~\ref{as:uniform-convergence}, Lemma~\ref{lem:dual-value}, \eqref{eq:statistical-primal-dual-gap-proof}, and \eqref{eq:complementarity-inner-proof} yields
\begin{align}
\left|
    \E_{D_0}
    \left[
        \ell_0(
            \widehat f_{\theta_\Phi},z
        )
    \right]
    -
    P^\star
\right|
\leq\;&
\zeta_0
+
\Delta_D
+
L
\left(
    1+
    \frac{B_0}{s}
\right)
\nu_{\mathcal H}
\nonumber\\
&+
\left(
    \Gamma
    +
    \frac{B\mu}{L^2}
\right)
\Delta_F
=
\Delta_O,
\end{align}
which proves \eqref{eq:statistical-objective}.

\paragraph{Union Bound.}
By a union bound, these events hold simultaneously with probability at least $1-4\delta$.

\end{proof}
\subsection{Proof of Theorem~\ref{thm:empirical-sensitivity}}
\label{apx:sensitivity-thm}

\begin{proof}

For the perturbed problem~\eqref{P:primal_perturbed}, the Lagrangian is
\begin{equation}
L_u(f_\theta,\lambda)
=
L(f_\theta,\lambda)
-
\E_D[\lambda(z)u(z)],
\end{equation}
and therefore its dual function satisfies
\begin{equation}
g_u(\lambda)
=
g(\lambda)
-
\E_D[\lambda(z)u(z)].
\end{equation}
By weak duality, for every $\lambda\geq0$,
\begin{equation}
P(u)
\geq
g(\lambda)
-
\E_D[\lambda(z)u(z)].
\end{equation}
Evaluating at the learned multiplier $\hat\lambda^\star_\Phi$ gives
\begin{align}
P(u)
&\geq
P(0)
-
\E_D[\hat\lambda^\star_\Phi(z)u(z)]
-
\left[
P(0)-g(\hat\lambda^\star_\Phi)
\right].
\label{eq:sensitivity-basic}
\end{align}

The remaining term can be decomposed as
\begin{align}
P(0)-g(\hat\lambda^\star_\Phi)
=
&
\left[
P(0)-D^\star
\right]
+
\left[
D^\star-\widehat D_\Phi^\star
\right]
\nonumber\\
&
+
\left[
\widehat g(\hat\lambda^\star_\Phi)
-
g(\hat\lambda^\star_\Phi)
\right],
\label{eq:sensitivity-decomposition}
\end{align}
where we used
\begin{equation}
\widehat g(\hat\lambda^\star_\Phi)
=
\widehat D_\Phi^\star.
\end{equation}

From the primal--dual approximation result used in
Theorem~\ref{thm:primal-recovery},
\begin{equation}
P(0)-D^\star
\leq
L\left(
1+\frac{B_0}{s}
\right)\nu_{\mathcal H}.
\end{equation}
Lemma~\ref{lem:dual-value} gives
\begin{equation}
D^\star-\widehat D_\Phi^\star
\leq
B\nu_\Phi+\zeta_0+\Gamma\zeta.
\end{equation}
Finally, Proposition~\ref{prop:derived_uc} gives
$\Lambda_\Phi$ and $\mathcal H$ gives
\begin{equation}
\left|
\widehat g(\hat\lambda^\star_\Phi)
-
g(\hat\lambda^\star_\Phi)
\right|
\leq
\zeta_0+\Gamma\zeta+B\zeta_\Lambda.
\end{equation}
Substituting these bounds into~\eqref{eq:sensitivity-decomposition}
and then into~\eqref{eq:sensitivity-basic}
yields~\eqref{eq:empirical-sensitivity}.

\end{proof}
\subsection{Extension to the Augmented Lagrangian}
\label{app:augmented}

The analysis in Section~\ref{sec:param} was developed for the standard Lagrangian. In this section, we show that the same arguments extend to the Powell--Hestenes--Bertsekas (PHB) augmented Lagrangian. The main difference is that augmentation directly provides smoothness of the dual function. Consequently, primal recovery is controlled by the augmentation parameter $\alpha$, rather than exclusively by the curvature of the primal objective.

For $\alpha>0$, define the PHB penalty
\begin{equation}
\psi_\alpha(t,\lambda)
:=
\frac{
[\lambda+\alpha t]_+^2-\lambda^2
}{2\alpha},
\qquad \lambda\geq 0,
\label{eq:phb_penalty}
\end{equation}
and the corresponding statistical augmented Lagrangian
\begin{equation}
\mathcal L_\alpha(f_\theta,\lambda)
:=
\mathbb E_{\mathcal D_0}
[\ell_0(f_\theta,z)]
+
\mathbb E_{\mathcal D}
\left[
\psi_\alpha
\big(
\ell(f_\theta,z),\lambda(z)
\big)
\right].
\label{eq:stat_aug_lagrangian}
\end{equation}
Its augmented dual function and optimal value are
\begin{equation}
g_\alpha(\lambda)
:=
\min_{f_\theta\in\mathcal H}
\mathcal L_\alpha(f_\theta,\lambda),
\qquad
D_\alpha^\star
:=
\sup_{\lambda\in L_1^+(\mathcal D)}
g_\alpha(\lambda).
\label{eq:stat_aug_dual}
\end{equation}

Likewise, from samples $\mathcal S_0$ and $\mathcal S$, we define
\begin{equation}
\widehat{\mathcal L}_\alpha(f_\theta,\lambda)
:=
\frac{1}{N_0}
\sum_{z_n\in\mathcal S_0}
\ell_0(f_\theta,z_n)
+
\frac{1}{N}
\sum_{z_n\in\mathcal S}
\psi_\alpha
\big(
\ell(f_\theta,z_n),\lambda(z_n)
\big),
\label{eq:emp_aug_lagrangian}
\end{equation}
with
\begin{equation}
\widehat g_\alpha(\lambda)
:=
\min_{f_\theta\in\mathcal H}
\widehat{\mathcal L}_\alpha(f_\theta,\lambda).
\end{equation}
Restricting again the multiplier to $\Lambda_\Phi$, the parametrized empirical augmented dual problem is
\begin{equation}
\widehat\lambda_{\Phi,\alpha}^\star
\in
\argmax_{\lambda_\phi\in\Lambda_\Phi}
\widehat g_\alpha(\lambda_\phi),
\qquad
\widehat D_{\Phi,\alpha}^\star
:=
\widehat g_\alpha
(
\widehat\lambda_{\Phi,\alpha}^\star
).
\label{eq:param_aug_dual}
\end{equation}

We use the same assumptions as in the main text. For compactness, let $\zeta_\alpha(N,\delta)$ denote the uniform-convergence error associated with the augmented constraint term in \eqref{eq:emp_aug_lagrangian}. Since $\psi_\alpha(t,\lambda)$ is $(\Gamma+\alpha B)$-Lipschitz in $t$ over the bounded domain considered here, the same Rademacher-contraction argument used for Assumption~1 gives a rate of order
\begin{equation}
\zeta_\alpha(N,\delta)
=
O\big((\Gamma+\alpha B)\zeta(N,\delta)\big).
\end{equation}
Moreover, $\psi_\alpha$ is Lipschitz in the multiplier with constant
\begin{equation}
\overline B_\alpha
:=
\max\left\{B,\frac{\Gamma}{\alpha}\right\}.
\label{eq:B_aug}
\end{equation}

\begin{lemma}[Augmented dual learnability]
\label{lem:aug_dual_learnability}
Let Assumptions~1 and~2 hold, with $\nu_\Phi$ defined with respect to an optimal augmented multiplier $\lambda_\alpha^\star$. Then, with probability at least $1-3\delta$,
\begin{equation}
\left|
D_\alpha^\star
-
\widehat D_{\Phi,\alpha}^\star
\right|
\leq
\overline B_\alpha\nu_\Phi
+
\zeta_0(N_0,\delta)
+
\zeta_\alpha(N,\delta)
=:
\Delta_D^\alpha .
\label{eq:aug_dual_learning}
\end{equation}
\end{lemma}

Lemma~\ref{lem:aug_dual_learnability} is the direct augmented counterpart of Lemma~1. The only changes are the Lipschitz constant $\overline B_\alpha$ associated with the PHB penalty and the corresponding uniform-convergence rate $\zeta_\alpha$.

We next state the primal recovery result. Define
\begin{equation}
C_\alpha
:=
1+\Gamma+\alpha B
\end{equation}
and
\begin{equation}
\varepsilon_\alpha
:=
\Delta_D^\alpha
+
\zeta_0
+
L C_\alpha\nu_{\mathcal H}.
\label{eq:eps_aug}
\end{equation}
Further define
\begin{equation}
\Delta_R^\alpha
:=
\sqrt{
\frac{2}{\alpha}
\varepsilon_\alpha
}
+
\sqrt{
\frac{
2L^3C_\alpha\nu_{\mathcal H}
}{\mu}
}.
\label{eq:aug_residual}
\end{equation}

\begin{theorem}[Statistical primal recovery with augmentation]
\label{thm:aug_primal_recovery}
Let the assumptions of Theorem~1 hold, and let
\begin{equation}
\widehat f_{\theta_{\Phi,\alpha}}
\in
\argmin_{f_\theta\in\mathcal H}
\widehat{\mathcal L}_\alpha
\left(
f_\theta,
\widehat\lambda_{\Phi,\alpha}^\star
\right).
\end{equation}
Then, with probability at least $1-4\delta$,
\begin{equation}
\mathbb E_{z\sim\mathcal D}
\left[
[\ell(\widehat f_{\theta_{\Phi,\alpha}},z)]_+
\right]
\leq
\Delta_R^\alpha
+
\zeta(N,\delta)
=:
\Delta_F^\alpha,
\label{eq:aug_feasibility}
\end{equation}
\begin{equation}
\left|
\mathbb E_{\mathcal D_0}
\left[
\ell_0(
\widehat f_{\theta_{\Phi,\alpha}},z
)
\right]
-
P^\star
\right|
\leq
\zeta_0
+
\Delta_D^\alpha
+
L\left(1+\frac{B_0}{s}\right)\nu_{\mathcal H}
\
+
\left(
\Gamma
+
\frac{\alpha}{2}
\max\left\{
B,\frac{\Gamma}{\alpha}
\right\}
\right)
\Delta_R^\alpha .
\label{eq:aug_optimality}
\end{equation}
\end{theorem}

The sensitivity result extends in the same way. We additionally assume, as standard for the PHB formulation, that an optimal augmented multiplier is also a Lagrange multiplier of the original constrained problem.

\begin{theorem}[Empirical approximate sensitivity with augmentation]
\label{thm:aug_sensitivity}
Let the assumptions of Theorem~\ref{thm:aug_primal_recovery} and Assumption~7 hold. Then, with probability at least $1-4\delta$, simultaneously for every admissible perturbation $u$,
\begin{equation}
P(u)
\geq
P(0)
-
\mathbb E_{z\sim\mathcal D}
\left[
\widehat\lambda_{\Phi,\alpha}^\star(z)u(z)
\right]
\
-
\Bigg[
L\left(1+\frac{B_0}{s}\right)\nu_{\mathcal H}
+
\overline B_\alpha\nu_\Phi
+
2\zeta_0
+
2\zeta_\alpha
+
\overline B_\alpha\zeta_\Lambda
\Bigg].
\label{eq:aug_sensitivity}
\end{equation}
\end{theorem}

\paragraph{Comparison with the standard Lagrangian.}
The augmented and non-augmented guarantees have the same qualitative structure: the recovery error is determined by statistical estimation, approximation of the optimal multiplier by $\Lambda_\Phi$, and approximation of the convexified predictor by the original hypothesis class. The main difference appears in the conversion from dual suboptimality to primal feasibility.

For the standard Lagrangian, strong convexity of the objective implies that the convexified dual has gradient Lipschitz constant $L^2/\mu$, which produces the term
\begin{equation}
\sqrt{
\frac{2L^2}{\mu}
\varepsilon
}.
\end{equation}
In contrast, the PHB augmented dual is $1/\alpha$-smooth. The corresponding term is therefore
\begin{equation}
\sqrt{
\frac{2}{\alpha}
\varepsilon_\alpha
}.
\end{equation}
Thus, augmentation replaces the dependence on the curvature of the primal objective by a dependence on the augmentation parameter $\alpha$. In particular, larger $\alpha$ directly improves the conversion of dual suboptimality into constraint satisfaction. When $\nu_{\mathcal H}=0$, this removes the dependence on $\mu$ from the feasibility argument entirely.

There is nevertheless a tradeoff: increasing $\alpha$ also increases the Lipschitz constant of the augmented loss, and therefore the terms $C_\alpha$ and $\zeta_\alpha$. Hence, $\alpha$ improves the conditioning of the dual recovery step while potentially increasing approximation and statistical errors associated with the augmented objective.

\subsubsection{Proofs}
\label{app:augmented_proofs}

\paragraph{Proof of Lemma~\ref{lem:aug_dual_learnability}.}
The proof follows that of Lemma~1. The only modification is the comparison between two multipliers. From \eqref{eq:phb_penalty},
\begin{equation}
\left|
\frac{\partial}{\partial\lambda}
\psi_\alpha(t,\lambda)
\right|
=
\left|
\frac{
[\lambda+\alpha t]_+-\lambda
}{\alpha}
\right|
\leq
\overline B_\alpha .
\end{equation}
Hence,
\begin{equation}
\left|
\psi_\alpha(
\ell(f_\theta,z),
\lambda_\phi(z)
)
-
\psi_\alpha(
\ell(f_\theta,z),
\lambda_\alpha^\star(z)
)
\right|
\leq
\overline B_\alpha
|
\lambda_\phi(z)-\lambda_\alpha^\star(z)
|.
\end{equation}
Integrating gives the term $\overline B_\alpha\nu_\Phi$. The population-to-empirical comparison is identical after replacing the weighted constraint-loss uniform-convergence term $\Gamma\zeta$ by $\zeta_\alpha$. The remainder of the proof is unchanged.

\paragraph{Proof of Theorem~\ref{thm:aug_primal_recovery}.}
The proof follows the convexification argument of Theorem~1. The key modification is the dual-smoothness step. Define the PHB residual
\begin{equation}
r_\alpha(f,\lambda)
:=
\frac{
[\lambda+\alpha\ell(f)]_+
-
\lambda
}{\alpha}.
\label{eq:phb_residual}
\end{equation}
For the convexified empirical augmented dual,
\begin{equation}
\nabla
\widehat d_{C,\alpha}(\lambda)
=
r_\alpha(
\widehat f_{C,\alpha}(\lambda),
\lambda
).
\end{equation}
The PHB augmented dual is $1/\alpha$-smooth. Therefore, the cocoercivity argument used in the proof of Theorem~1 now gives
\begin{equation}
\left|
r_\alpha(
\widehat f_{C,\alpha}
(
\widehat\lambda_{\Phi,\alpha}^\star
),
\widehat\lambda_{\Phi,\alpha}^\star
)
\right|_{2,N}
\leq
\sqrt{
\frac{2}{\alpha}
\varepsilon_\alpha
}.
\label{eq:aug_cocoercivity}
\end{equation}
Moreover, coordinatewise,
\begin{equation}
[\ell(f,z)]_+
=
[r_\alpha(f,\lambda)(z)]_+,
\end{equation}
so \eqref{eq:aug_cocoercivity} directly controls empirical constraint violation.

The transfer from $\operatorname{conv}(\mathcal H)$ to $\mathcal H$ proceeds exactly as in Theorem~1. The only change is that the augmented Lagrangian is $LC_\alpha$-Lipschitz with respect to the approximation error, replacing the factor $L(1+\Gamma)$. Strong convexity then yields the second term in \eqref{eq:aug_residual}. Uniform convergence of the positive-part constraint loss completes \eqref{eq:aug_feasibility}.

For optimality, use the identity
\begin{equation}
\psi_\alpha(t,\lambda)
=
\lambda r
+
\frac{\alpha}{2}r^2,
\qquad
r
=
\frac{
[\lambda+\alpha t]_+-\lambda
}{\alpha}.
\end{equation}
The augmented contribution is therefore controlled by the residual bound \eqref{eq:aug_residual}. The remaining population--empirical and convexification terms are identical to those in the proof of Theorem~1.

\paragraph{Proof of Theorem~\ref{thm:aug_sensitivity}.}
The sensitivity argument is unchanged because an optimal PHB multiplier is also a multiplier of the original constrained problem and therefore has the same sensitivity interpretation. The proof of Theorem~2 can consequently be repeated verbatim after replacing Lemma~1 by Lemma~\ref{lem:aug_dual_learnability}. In particular, the terms $B\nu_\Phi+\Gamma\zeta$ appearing in the standard-dual argument are replaced by $\overline B_\alpha\nu_\Phi+\zeta_\alpha$. The dual-class generalization term $B\zeta_\Lambda$ is unchanged.

\section{Tasks Details }
\label{apx:c_exper_det}

\subsection{MNIST image classification}
\label{app:task:mnist}

\paragraph{Problem formulation.}
We first consider constrained image classification on MNIST. Rather than minimizing a conventional classification loss, the objective penalizes only the norm of the model parameters,
\begin{equation}
    f(\theta)
    =
    \frac{c}{2}\|\theta\|_2^2,
\end{equation}
while correct classification is imposed through pointwise margin constraints.

Let $z_j(x_i;\theta)$ denote the logit assigned to class $j$ for image $x_i$, and let $y_i$ denote its ground-truth class. For every incorrect class $j\neq y_i$, we impose
\begin{equation}
    g_{ij}(\theta)
    :=
    \epsilon
    -
    \left(
        z_{y_i}(x_i;\theta)
        -
        z_j(x_i;\theta)
    \right)
    \leq 0.
\end{equation}
Thus, feasibility requires the correct-class logit to exceed every incorrect-class logit by at least the prescribed margin $\epsilon$. We use
\begin{equation}
    \epsilon = 1,
    \qquad
    c = 0.1.
\end{equation}

MNIST contains
\begin{equation}
    N = 60{,}000
\end{equation}
training images. Since each image generates one constraint for each of its nine incorrect classes, the total number of pointwise constraints is
\begin{equation}
    M
    =
    60{,}000 \times 9
    =
    540{,}000.
\end{equation}

\paragraph{Primal model.}
The classifier is a convolutional neural network composed of two $3\times3$ convolutional layers of width $64$, followed by ReLU activations and max pooling, a fully connected hidden layer of width $32$, and a linear output classifier. The resulting primal network contains
\begin{equation}
    p = 138{,}282
\end{equation}
trainable parameters. Batch normalization is not used.

\paragraph{Dual parameterization.}
For MNIST, the parametric multiplier is implemented as a second convolutional
network with the same architecture as the primal classifier. In particular, the
dual network consists of two $3\times3$ convolutional layers of width $64$ with
ReLU activations and max pooling, followed by a fully connected hidden layer of
width $32$ and a linear output layer. Given an image $x_i$, the network produces
one multiplier for each class,
\begin{equation}
    h_\phi(x_i)\in\mathbb{R}^{10}.
\end{equation}
The outputs are mapped to nonnegative bounded multipliers according to
\begin{equation}
    \lambda_\phi(x_i)
    =
    \Gamma\,
    \sigma\!\left(h_\phi(x_i)\right),
\end{equation}
where $\Gamma$ is the multiplier cap and $\sigma$ is the logistic sigmoid. Since
the true class does not define a margin constraint, its corresponding component
is fixed to zero,
\begin{equation}
    [\lambda_\phi(x_i)]_{y_i}=0.
\end{equation}
The output-layer bias is initialized to $-7$, so that the initial multiplier
predictions are close to zero.
\subsection{FLORA-Bench workflow prediction}
\label{app:task:flora}

\paragraph{Problem formulation.}
We next consider the Coding-GD split of FLORA-Bench, where a graph neural network predicts whether a multi-agent workflow successfully solves a coding task. Each observation corresponds to a workflow--task pair.

The objective encourages successful workflows to rank above unsuccessful workflows within each task. Let $\mathcal{P}_t$ and $\mathcal{N}_t$ denote the successful and unsuccessful workflows associated with task $t$, respectively. We use the pairwise ranking objective
\begin{equation}
    f(\theta)
    =
    \frac{1}{|\mathcal{T}|}
    \sum_{t\in\mathcal{T}}
    \frac{1}{|\mathcal{P}_t||\mathcal{N}_t|}
    \sum_{i\in\mathcal{P}_t}
    \sum_{j\in\mathcal{N}_t}
    \log
    \left(
        1+
        \exp\left(
            -\left[
                s_i(\theta)-s_j(\theta)
            \right]
        \right)
    \right),
\end{equation}
where $s_i(\theta)$ denotes the predicted logit for workflow--task pair $i$.

The pointwise constraints additionally require every individual pair to be classified with sufficiently small binary cross-entropy,
\begin{equation}
    g_i(\theta)
    :=
    \operatorname{BCE}_i(\theta)
    -
    \tau
    \leq 0,
\end{equation}
with
\begin{equation}
    \tau = 0.1.
\end{equation}
The training set contains
\begin{equation}
    N = 24{,}489
\end{equation}
observed workflow--task pairs.

\paragraph{Dataset structure.}
The data form a partially observed task--workflow matrix with
\begin{equation}
    N_{\mathrm{task}} = 57,
    \qquad
    N_{\mathrm{workflow}} = 1026.
\end{equation}
The complete grid therefore contains
\begin{equation}
    57\times 1026
    =
    58{,}482
\end{equation}
possible pairs, of which approximately
\begin{equation}
    \frac{24{,}489}{58{,}482}
    \approx
    0.419
\end{equation}
are observed. Each workflow appears under multiple tasks,
\begin{equation}
    14
    \leq
    N_{\mathrm{tasks/workflow}}
    \leq
    51,
\end{equation}
and each task contains between
\begin{equation}
    221
    \leq
    N_{\mathrm{workflows/task}}
    \leq
    616
\end{equation}
workflows.

\paragraph{Primal model.}
The primal predictor is a graph neural network with
\begin{equation}
    p = 26{,}755{,}073
\end{equation}
trainable parameters. The primal architecture is held fixed across all experiments that vary the capacity of the dual model.

\paragraph{Dual parameterization.}
Unlike MNIST, the dual predictor does not operate directly on the raw workflow graph. Instead, it is conditioned on the penultimate representation produced by the primal GNN. Let
\begin{equation}
    h_\theta(w,t)
    \in
    \mathbb{R}^{d_h}
\end{equation}
denote the primal representation associated with workflow $w$ and task $t$. The multiplier predictor takes this representation as input,
\begin{equation}
    \lambda_\phi(w,t)
    =
    q_\phi
    \left(
        h_\theta(w,t)
    \right).
\end{equation}
Since each observed workflow--task pair has a single BCE constraint, the dual output is scalar,
\begin{equation}
    \lambda_\phi(w,t)
    \in
    \mathbb{R}_{+}.
\end{equation}

The same primal representation is used throughout the capacity experiments, so changing the dual-network width modifies only the expressivity of the multiplier class and not the primal model or its features.
\subsection{TOFU language-model unlearning}
\label{app:task:tofu}

\paragraph{Problem formulation.}
We next study constrained language-model unlearning using the TOFU benchmark. The starting model is Llama-3.2-1B, previously finetuned on
\begin{equation}
    N_{\mathrm{FT}} = 4{,}000
\end{equation}
question--answer pairs describing fictitious authors. The forget set contains
\begin{equation}
    N_{\mathrm{forget}} = 400
\end{equation}
question--answer pairs associated with
\begin{equation}
    N_{\mathrm{authors}} = 20
\end{equation}
authors.

The primal objective preserves performance on a fixed retain set by minimizing its negative log-likelihood,
\begin{equation}
    f(\theta)
    =
    \frac{1}{|\mathcal{R}|}
    \sum_{(x,y)\in\mathcal{R}}
    -\log p_\theta(y\mid x),
\end{equation}
using a retain set of
\begin{equation}
    |\mathcal{R}| = 256
\end{equation}
sequences.

For each forget example $i$, we impose four pointwise constraints. Let
$s_i(\theta)$ denote the length-normalized answer log-probability ratio relative
to the reference model. The first two constraints impose upper and lower bounds
on this score,
\begin{equation}
    s_i(\theta) \leq u_i,
    \qquad
    s_i(\theta) \geq l_i,
\end{equation}
where the bounds are constructed from a per-example target with the prescribed
shift and band. In addition, two constraints control the token-level
log-probability profile by imposing an upper bound on the average
log-probability of the most likely answer tokens and a lower bound on the
average log-probability of the least likely answer tokens. Thus, each forget
example contributes four scalar pointwise constraints. The reported experiments
use a shift and band of $0.25$ and an optimization budget of $20$ epochs.

\paragraph{Dual parameterization.}
Each of the four constraints associated with a forget example has its own
nonnegative multiplier. The parametric formulation replaces the independent
tabular multipliers with a shared dual network,
\begin{equation}
    \lambda_\phi(x_i)
    \in
    \mathbb{R}_{+}^{4}.
\end{equation}
The dual network takes a pooled hidden representation of the forget example and
outputs one multiplier for each constraint,
\begin{equation}
    \lambda_\phi(x_i)
    =
    \left[
        \lambda_{\phi,1}(x_i),
        \lambda_{\phi,2}(x_i),
        \lambda_{\phi,3}(x_i),
        \lambda_{\phi,4}(x_i)
    \right]^\top.
\end{equation}
This construction allows the learned dual to predict the corresponding
multipliers for forget examples that were not used to fit the dual network.

\paragraph{Evaluation metrics.}
In addition to the retain objective and pointwise forgetting constraints, we
report the standard TOFU metrics for memorization, privacy, and utility. Their
aggregate is computed as a harmonic mean,
\begin{equation}
    S_{\mathrm{agg}}
    =
    \operatorname{HMean}
    \left(
        S_{\mathrm{mem}},
        S_{\mathrm{priv}},
        S_{\mathrm{util}}
    \right).
\end{equation}
Reporting these components separately is important because similar aggregate
scores can result from qualitatively different unlearning behaviors.
\subsection{IEEE30 Optimal Power flow dispatch}
\label{app:task:opf}

\paragraph{Problem formulation.}
Optimal power flow (OPF) is the central optimization problem underlying economic dispatch in electrical grids. It is repeatedly solved to allocate generation, determine operating points, and support electricity-market and infrastructure decisions~\cite{cain2012history}. We represent the grid as a network of $N$ buses connected by a set $\mathcal{E}$ of $M$ directed branches~\cite{chatzivasileiadis2018lecture}. Each bus $i$ is associated with complex voltage $v_i$, complex generation $s_i$, and demand $r_i$.

Under the standard $\pi$-section branch model~\cite{coffrin2018powermodels,zimmerman2011matpower}, the complex power flows through branch $(i,j)$ are
\begin{equation}
\begin{aligned}
    f_{ij}
    &=
    \left(y_{ij}+y^C_{ij}\right)^{*}
    \left|\frac{v_i}{t_{ij}}\right|^2
    -
    y_{ij}^{*}
    \frac{v_i v_j^{*}}{t_{ij}}, \\
    f_{ji}
    &=
    \left(y_{ij}+y^C_{ji}\right)^{*}
    |v_j|^2
    -
    y_{ij}^{*}
    \frac{v_i^{*}v_j}{t_{ij}^{*}},
\end{aligned}
\label{eq:opf-flows}
\end{equation}
where $t_{ij}$ denotes the transformer ratio, $y_{ij}$ the line admittance, and $y^C_{ij}$ and $y^C_{ji}$ the shunt admittances.

The generation cost is quadratic in the real generated power,
\begin{equation}
    C(s)
    =
    \sum_{i=1}^{N}
    c_{0i}
    +
    c_{1i}\,\mathrm{Re}(s_i)
    +
    c_{2i}\,\mathrm{Re}^2(s_i).
\label{eq:opf-cost}
\end{equation}
Given a demand vector
\begin{equation}
    r
    =
    [r_1;\ldots;r_N],
\end{equation}
the OPF problem minimizes generation cost while enforcing power balance together with generation, voltage, thermal, and phase-angle limits,
\begin{equation}
\begin{aligned}
    \min_{s,v}\quad
    & C(s) \\
    \mathrm{s.t.}\quad
    & s_i-r_i
    =
    \sum_{j\in n(i)} f_{ij}
    -
    \left(y_i^S\right)^*|v_i|^2, \\
    & s^G_{i,\min}
    \leq
    s_i
    \leq
    s^G_{i,\max}, \\
    & v_{i,\min}
    \leq
    |v_i|
    \leq
    v_{i,\max}, \\
    & |f_{ij}|
    \leq
    f_{ij,\max},
    \qquad
    |f_{ji}|
    \leq
    f_{ji,\max}, \\
    & \theta_{ij,\min}
    \leq
    \angle(v_i v_j^*)
    \leq
    \theta_{ij,\max},
\end{aligned}
\label{eq:opf-full}
\end{equation}
with branch flows defined by~\eqref{eq:opf-flows}. Collecting the grid parameters into $\mathcal{W}$, we write the problem compactly as
\begin{equation}
    \min_{s,v}
    \;
    C(s)
    \qquad
    \mathrm{s.t.}
    \qquad
    g(s,v;r,\mathcal{W})\leq 0,
    \qquad
    h(s,v;r,\mathcal{W})=0.
\label{eq:opf-compact}
\end{equation}

The AC-OPF problem is nonconvex~\cite{lesieutre2005convexity} and NP-hard in general~\cite{bienstock2019nphard}. Classical interior-point methods and convex relaxations can provide accurate solutions~\cite{molzahn2016illustrative,low2014convex,kocuk2016socp}, but repeatedly solving OPF as demand changes can be computationally expensive~\cite{castillo2013computational}. This motivates learning an amortized predictor
\begin{equation}
    [s,v]
    =
    \Phi(r;\mathcal{A},\mathcal{W}),
\end{equation}
typically parameterized by a graph neural network exploiting the grid topology~\cite{owerko2020gnn,owerko2024unsupervised}.

\paragraph{Pointwise constrained learning.}
Supervised approaches regress onto solutions produced by an optimization solver~\cite{guha2019ml,owerko2020gnn}, whereas unsupervised methods directly optimize the OPF objective and physical constraints~\cite{owerko2024unsupervised}. Constraint penalties or expectation-based constrained formulations can nevertheless allow violations on individual demand realizations to be compensated by slack on others~\cite{huang2022deepopfv,pan2023deepopf,wang2021gradient,huang2024unsupervised,fioretto2020predicting,chen2022unsupervised,chamon2020pac}.

Following~\cite{owerko2025pointwise}, we instead consider the pointwise-constrained learning problem
\begin{equation}
\begin{aligned}
    \min_{\mathcal{A}}\quad
    &
    \mathbb{E}_{r\sim\rho}
    \left[
        C\left(
            \Phi(r;\mathcal{A},\mathcal{W})
        \right)
    \right] \\
    \mathrm{s.t.}\quad
    &
    g\left(
        \Phi(r;\mathcal{A},\mathcal{W});
        r,\mathcal{W}
    \right)
    \leq 0,
    \\
    &
    h\left(
        \Phi(r;\mathcal{A},\mathcal{W});
        r,\mathcal{W}
    \right)
    =
    0,
    \qquad
    \forall r.
\end{aligned}
\label{eq:opf-pointwise}
\end{equation}
This formulation associates each demand realization with its own inequality and equality multipliers and can be solved through an augmented-Lagrangian primal--dual procedure~\cite{wierzbicki1977projection,owerko2025pointwise}. The resulting multipliers additionally admit a sensitivity interpretation, quantifying how strongly the optimal dispatch cost depends on the corresponding physical constraints~\cite{boyd2004convex}.

\paragraph{Dataset and primal model.}
We evaluate on the IEEE 30-bus system, which contains
\begin{equation}
    N_{\mathrm{bus}}=30,
    \qquad
    N_{\mathrm{gen}}=6,
    \qquad
    N_{\mathrm{branch}}=41.
\end{equation}
The dataset contains
\begin{equation}
    N_{\mathrm{sample}}
    =
    8{,}000
\end{equation}
demand realizations. The primal predictor is the graph neural network used for pointwise OPF learning in~\cite{owerko2025pointwise}, which maps each demand realization to the corresponding generation and voltage variables.

\paragraph{Dual structure.}
In the pointwise formulation, every demand realization has an independent structured vector of multipliers. The complete tabular representation contains five entries per bus, four per generator, and six per branch, resulting in
\begin{equation}
    M_{\mathrm{tab}}
    =
    5N_{\mathrm{bus}}
    +
    4N_{\mathrm{gen}}
    +
    6N_{\mathrm{branch}}
    =
    420
\end{equation}
multipliers per operating condition.

The first
\begin{equation}
    M_{\mathrm{eq}}
    =
    90
\end{equation}
entries correspond to equality constraints and are therefore free-sign, while the remaining entries correspond to inequality constraints and are nonnegative. Storing one such vector independently for every training realization causes the number of dual variables to grow linearly with the size of the dataset~\cite{park2023selfsupervised,owerko2025pointwise}.

\paragraph{Dual parameterization.}
We replace this per-realization multiplier table with a shared dual network. Rather than maintaining an independent multiplier vector for every demand profile, the parametric dual reads the latent node representations produced by the primal GNN and predicts the corresponding multipliers for buses, generators, and transmission branches.

Let
\begin{equation}
    h_\theta(r)
    \in
    \mathbb{R}^{d_h}
\end{equation}
denote the penultimate representation of the primal GNN for demand realization $r$. The parameterized multiplier is obtained as
\begin{equation}
    \lambda_\phi(r)
    =
    q_\phi
    \left(
        h_\theta(r)
    \right).
\end{equation}
The output preserves the physical grouping of the OPF constraints, so that separate predictions are produced for the multiplier components associated with buses, generators, and branches.

The parametric implementation outputs
\begin{equation}
    M_{\mathrm{param}}
    =
    338
\end{equation}
multipliers per operating condition. The first
\begin{equation}
    90
\end{equation}
components correspond to free-sign equality multipliers, while the remaining outputs are constrained to be nonnegative. The parametric representation omits one two-dimensional branch constraint group corresponding to voltage-angle differences. Consequently, comparisons between the tabular and parametric duals are performed on their common inequality subset,
\begin{equation}
    \mathcal{I}
    =
    \{90,\ldots,337\},
\end{equation}
which contains
\begin{equation}
    |\mathcal{I}|
    =
    248
\end{equation}
inequality multipliers per operating condition.

This construction changes the scaling of the dual representation. Whereas the tabular formulation stores a distinct multiplier vector for every training sample, the parametric formulation stores only the parameters $\phi$ of the shared dual network. More importantly for our setting, it defines
\begin{equation}
    \lambda_\phi(r)
\end{equation}
for previously unseen demand realizations, allowing sample-level constraint sensitivities to be evaluated beyond the training set.

\section{Implementation Details}
\label{apx:b_algo}

We now describe how the parametric dual problem is optimized in practice.
Given the constraint samples
\[
S=\{z_n\}_{n=1}^N,
\]
we write the empirical constraints as
\begin{equation}
    g_n(\theta)
    :=
    \ell(f_\theta,z_n),
    \qquad n=1,\ldots,N,
\end{equation}
where the constraint threshold has been absorbed into the definition of
$\ell$, consistently with the formulation in Section~\ref{sec:every_learn}.

For a constraint value $g$, multiplier $\lambda\geq 0$, and penalty
parameter $\rho>0$, we use the Powell--Hestenes--Rockafellar augmented
Lagrangian term
\begin{equation}
    \psi_\rho(g,\lambda)
    :=
    \frac{
        [\lambda+\rho g]_+^2-\lambda^2
    }{2\rho}.
    \label{eq:impl_phr}
\end{equation}
The empirical augmented Lagrangian associated with a parametric multiplier
$\lambda_\phi\in\Lambda_\Phi$ is therefore
\begin{equation}
    \widehat L_\rho(f_\theta,\lambda_\phi)
    :=
    \frac{1}{N_0}
    \sum_{z_n\in S_0}
    \ell_0(f_\theta,z_n)
    +
    \frac{1}{N}
    \sum_{z_n\in S}
    \psi_\rho
    \bigl(
        \ell(f_\theta,z_n),
        \lambda_\phi(z_n)
    \bigr).
    \label{eq:impl_parametric_alm}
\end{equation}

\paragraph{Pointwise and parametric dual updates.}
The standard pointwise formulation maintains an independent multiplier
$\lambda_n$ for every constraint and performs the projected ascent update
\begin{equation}
    \lambda_n^{+}
    =
    \Pi_{[0,\Gamma]}
    \left(
        \lambda_n+\rho\,\ell(f_\theta,z_n)
    \right),
    \label{eq:pointwise_dual_update}
\end{equation}
where $\Gamma$ is the multiplier bound introduced in
Section~\ref{sec:param}.

The parametric formulation replaces the collection
$\{\lambda_n\}_{n=1}^N$ with a single function
$\lambda_\phi\in\Lambda_\Phi$. At outer iteration $k$, we first evaluate
the current multiplier function on the training constraints,
\begin{equation}
    \lambda_{\phi_k}(z_n),
\end{equation}
and form the corresponding pointwise augmented-Lagrangian ascent targets,
\begin{equation}
    \widetilde{\lambda}_{n}^{\,k+1}
    =
    \Pi_{[0,\Gamma]}
    \left(
        \lambda_{\phi_k}(z_n)
        +
        \rho\,\ell(f_{\theta_k},z_n)
    \right).
    \label{eq:parametric_dual_target}
\end{equation}
When smoothing is used, we replace this target by
\begin{equation}
    \bar{\lambda}_{n}^{\,k+1}
    =
    \beta \lambda_{\phi_k}(z_n)
    +
    (1-\beta)
    \widetilde{\lambda}_{n}^{\,k+1},
    \qquad
    0\leq\beta<1.
    \label{eq:parametric_dual_target_smooth}
\end{equation}

The updated dual function is then obtained by approximately projecting these
pointwise ascent values back onto the parametric family,
\begin{equation}
    \phi_{k+1}
    \approx
    \argmin_{\phi}
    \frac{1}{|\mathcal I_{\mathrm{fit}}|}
    \sum_{n\in\mathcal I_{\mathrm{fit}}}
    \left|
        \lambda_\phi(z_n)
        -
        \bar{\lambda}_{n}^{\,k+1}
    \right|^2.
    \label{eq:parametric_dual_projection}
\end{equation}
Thus, the regression step is not used to fit a separately optimized
pointwise multiplier table. Rather, it implements an approximate projection
of an augmented-Lagrangian dual-ascent step back onto
$\Lambda_\Phi$. No persistent sample-wise multipliers are maintained:
the dual state of the parametric method is entirely represented by
$\phi$.

\paragraph{Primal update.}
For fixed $\phi_k$, the primal parameters are updated by approximately
minimizing the empirical augmented Lagrangian
\begin{equation}
    \theta_{k+1}
    \approx
    \argmin_{\theta}
    \widehat L_\rho(f_\theta,\lambda_{\phi_k}),
    \label{eq:parametric_primal_update}
\end{equation}
using stochastic gradient descent. Equivalently, for a minibatch
$\mathcal B\subseteq S$, the constraint contribution is
\begin{equation}
    \frac{1}{|\mathcal B|}
    \sum_{z_n\in\mathcal B}
    \frac{
        \left[
            \lambda_{\phi_k}(z_n)
            +
            \rho\,\ell(f_\theta,z_n)
        \right]_+^2
        -
        \lambda_{\phi_k}(z_n)^2
    }{2\rho}.
    \label{eq:parametric_primal_minibatch}
\end{equation}
Primal and dual updates are alternated throughout training.

\paragraph{Relation to the parametric dual problem.}
Equations~\ref{eq:parametric_dual_target}--\ref{eq:parametric_dual_projection}
can be interpreted as a practical approximate ascent procedure for the
parametric empirical dual problem. The first step performs the standard
augmented-Lagrangian ascent update at the observed constraints, while the
second projects the resulting multiplier values back onto the shared
function class $\Lambda_\Phi$. The projection is approximate because
$\lambda_\phi$ is represented by a neural network and
equation~\ref{eq:parametric_dual_projection} is solved by a finite number
of stochastic-gradient steps. Consequently, the theoretical results,
which are stated for an exact optimizer of the parametric empirical dual,
should be understood as characterizing the limiting optimization problem;
the implementation provides its stochastic approximate realization.

\paragraph{Held-out evaluation of the dual function.}
To evaluate whether $\lambda_\phi$ learns a transferable multiplier rule
rather than memorizing the observed constraints, we optionally partition
the constraint samples into a dual-fitting set
$\mathcal I_{\mathrm{fit}}$ and a held-out set
$\mathcal I_{\mathrm{test}}$. The primal optimization continues to use all
constraints, whereas equation~\ref{eq:parametric_dual_projection} is
optimized only over $\mathcal I_{\mathrm{fit}}$. Multiplier predictions on
$\mathcal I_{\mathrm{test}}$ therefore provide an out-of-sample evaluation
of the learned dual function.

For evaluation only, we compare these predictions against multipliers
obtained from an independently trained pointwise-dual reference. This
reference is not used as a regression target for the parametric method.
.

\begin{algorithm}
\caption{Parametric augmented-Lagrangian optimization}
\label{alg:parametric_dual_alm}
\begin{algorithmic}[1]
\Require Objective sample $S_0$, constraint sample $S=\{z_n\}_{n=1}^N$,
primal parameters $\theta$, dual parameters $\phi$, penalty $\rho$,
multiplier bound $\Gamma$, smoothing parameter $\beta$,
dual-fitting set $\mathcal I_{\mathrm{fit}}$
\Ensure Trained primal predictor $f_\theta$ and parametric multiplier $\lambda_\phi$

\State Initialize $\theta$ and $\phi$
\State Optionally warm up the primal parameters

\For{$k=1,2,\ldots$}

    \State \textbf{Primal update:}
    \For{each minibatch $\mathcal B\subseteq S$}
        \State Evaluate $\lambda_{\phi_k}(z_n)$ for $z_n\in\mathcal B$
        \State Update $\theta$ by descending a stochastic estimate of
        \Statex \hspace{1.5em}
        $\widehat L_\rho(f_\theta,\lambda_{\phi_k})$
    \EndFor

    \State \textbf{Dual ascent targets:}
    \For{$n\in\mathcal I_{\mathrm{fit}}$}
        \State
        $\displaystyle
        \widetilde{\lambda}_{n}^{\,k+1}
        \gets
        \Pi_{[0,\Gamma]}
        \left(
            \lambda_{\phi_k}(z_n)
            +
            \rho\,\ell(f_{\theta_{k+1}},z_n)
        \right)$
        \State
        $\displaystyle
        \bar{\lambda}_{n}^{\,k+1}
        \gets
        \beta\lambda_{\phi_k}(z_n)
        +(1-\beta)\widetilde{\lambda}_{n}^{\,k+1}$
    \EndFor

    \State \textbf{Projection onto $\Lambda_\Phi$:}
    \For{$j=1,\ldots,T_{\mathrm{dual}}$}
        \State Sample $\mathcal B_\lambda\subseteq\mathcal I_{\mathrm{fit}}$
        \State Update $\phi$ by descending
        \Statex \hspace{1.5em}
        $\displaystyle
        \frac{1}{|\mathcal B_\lambda|}
        \sum_{n\in\mathcal B_\lambda}
        \left|
            \lambda_\phi(z_n)
            -
            \bar{\lambda}_{n}^{\,k+1}
        \right|^2$
    \EndFor

    \State Optionally update $\rho$ according to the constraint-violation schedule

\EndFor
\end{algorithmic}
\end{algorithm}

\section{Additional Results}
\label{apx:d_add_res}

This section provides additional experimental results complementing those presented in the main text. We first report primal-recovery results for MNIST, FLORA-Bench, and a disaggregated analysis of the IEEE 30-bus experiment. We then provide additional capacity and sensitivity studies on MNIST and FLORA-Bench, respectively. We focus these more computationally intensive ablations on MNIST and FLORA-Bench because they permit repeated training and evaluation across multiple dual architectures and perturbation experiments at substantially lower computational cost. Together, these experiments further support the two properties studied in the main text: that parametric duals preserve the quality of the recovered primal solution and that their predictions retain meaningful information about sample-level sensitivity on unseen data.

\subsection{Additional Primal Recovery Results}
\label{app:additional_primal}

\paragraph{Primal recovery on MNIST.}
Figure~\ref{fig:mnist_primal} complements the language-model unlearning experiment in the main text with primal-recovery results on MNIST. The pointwise and parameterized dual formulations produce nearly identical constraint profiles, driving the constraint values close to the feasible region across almost the entire test distribution. At the same time, the two formulations attain comparable primal objective values and essentially identical classification performance. These results show that replacing independent sample-wise multipliers with a shared parametric function $\lambda_\phi(z)$ does not substantially alter the recovered primal model on MNIST.

\begin{figure}[t]
    \centering
    \includegraphics[width=\linewidth]{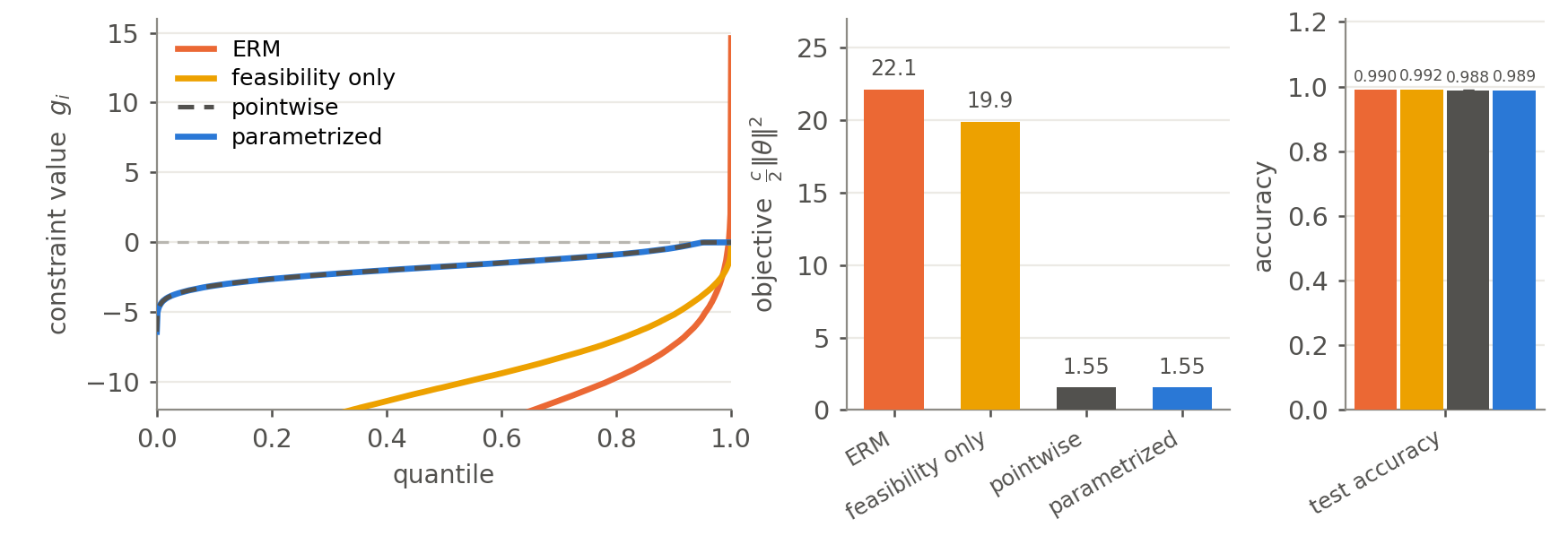}
    \caption{\textbf{Primal recovery on MNIST.}
    Comparison of the parameterized dual formulation with independent pointwise multipliers and the corresponding unconstrained or partially constrained baselines. The parameterized formulation closely matches the pointwise solution in constraint satisfaction and primal objective while preserving classification performance.}
    \label{fig:mnist_primal}
\end{figure}

\paragraph{Primal recovery on FLORA-Bench.}
A similar behavior is observed on FLORA-Bench, as shown in Figure~\ref{fig:flora_primal_app}. Both the pointwise and parameterized dual formulations drive the pointwise constraint violations close to zero and attain the same primal objective value, while maintaining essentially identical test accuracy. In contrast, ERM and feasibility-only training obtain substantially larger objective values and exhibit worse constraint satisfaction. Thus the parameterized formulation preserves the quality of the primal solution. 

\begin{figure}[t]
    \centering
    \includegraphics[width=\linewidth]{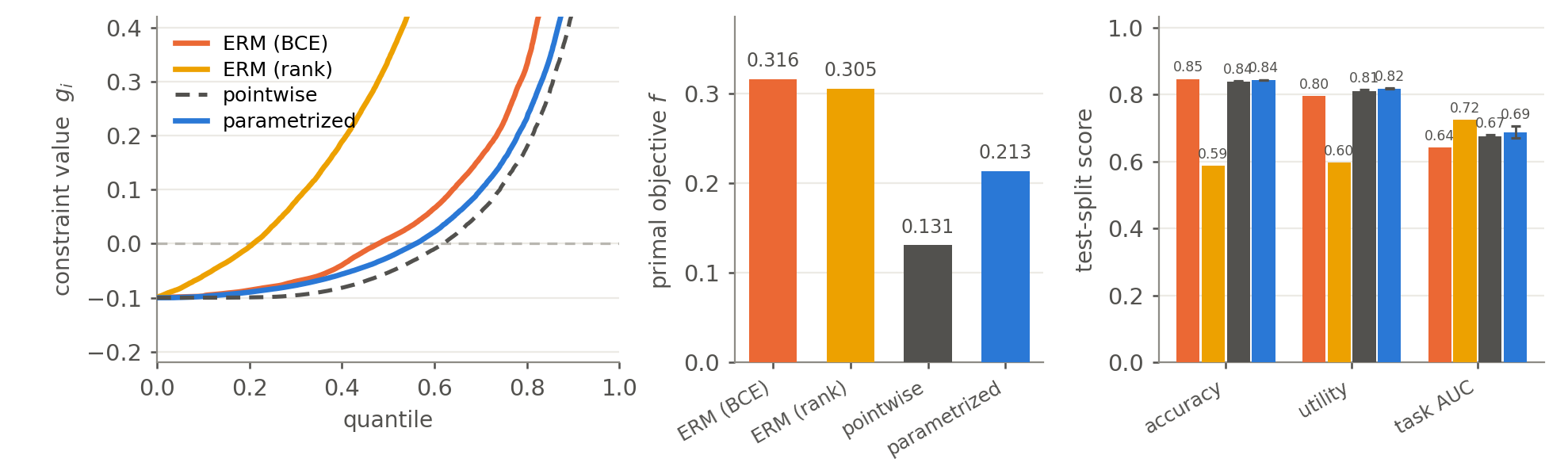}
    \caption{\textbf{Primal recovery on FLORA-Bench.}
    Comparison of ERM, feasibility-only training, independent pointwise multipliers, and parameterized dual multipliers. Pointwise and parameterized duals achieve near-zero constraint violations and comparable objective values while maintaining essentially identical test accuracy.}
    \label{fig:flora_primal_app}
\end{figure}

\paragraph{Constraint violations by type on IEEE 30-bus.}
Table~\ref{tab:ieee30-by-constraint} complements the aggregate IEEE 30-bus results reported in the main text by disaggregating violations according to constraint type. The parameterized formulation maintains small violations across the different physical quantities, despite learning a shared representation for hundreds of heterogeneous constraints. Its error profile differs across constraint classes: the largest deviations arise for reactive generation and branch-flow constraints, whereas active generation and voltage-magnitude violations remain small. This finer-grained analysis reinforces the aggregate result from the main text: the learned dual parameterization remains effective even when simultaneously representing constraints that act on different variables and different parts of the network.

\begin{figure}[t]
\centering
\resizebox{\linewidth}{!}{%
\begin{tabular}{llcccccccccc}
\toprule
& & \multicolumn{2}{c}{$\mathrm{Re}(\mathbf{s}^g)$ (\%)}
  & \multicolumn{2}{c}{$\mathrm{Im}(\mathbf{s}^g)$ (\%)}
  & \multicolumn{2}{c}{$|V|$ (\%)}
  & \multicolumn{2}{c}{$|\mathbf{f}_f|$ (\%)}
  & \multicolumn{2}{c}{$|\mathbf{f}_t|$ (\%)} \\
\cmidrule(lr){3-4} \cmidrule(lr){5-6} \cmidrule(lr){7-8}
\cmidrule(lr){9-10} \cmidrule(lr){11-12}
& & Mean & Max & Mean & Max & Mean & Max & Mean & Max & Mean & Max \\
\midrule
& Param
& 0.32 & 0.53
& 5.68 & 5.68
& 0.12 & 0.12
& 4.90 & 4.90
& 2.49 & 2.49 \\
& Dual-P
& 0.20 & 0.25
& 0.69 & 0.70
& 0.74 & 0.85
& 0.68 & 0.68
& 0.15 & 0.15 \\
& Dual-H
& 0.35 & 0.72
& 0.90 & 0.92
& 3.76 & 5.61
& 0.25 & 0.25
& 0.00 & 0.00 \\
& Dual-S
& 0.55 & 0.86
& 0.14 & 0.14
& 6.00 & 8.21
& 0.00 & 0.00
& 0.00 & 0.00 \\
\bottomrule
\end{tabular}%
}
\caption{\textbf{Constraint violations on IEEE 30-bus, disaggregated by constraint type.}
For each test operating point, the mean and maximum violation are computed separately for active generation $\mathrm{Re}(\mathbf{s}^g)$, reactive generation $\mathrm{Im}(\mathbf{s}^g)$, voltage magnitude $|V|$, and branch-flow magnitudes $|\mathbf{f}_f|$ and $|\mathbf{f}_t|$, and are then averaged across the test set. Violations are expressed as percentages of the corresponding feasible range. Voltage-angle constraints are omitted because no violations are observed. Baseline values are reproduced from~\cite{owerko2025pointwise}, Fig.~11.}
\label{tab:ieee30-by-constraint}
\end{figure}

\subsection{Additional Dual Capacity Results}
\label{app:additional_capacity}

\paragraph{Dual capacity on MNIST.}
We further study how the capacity of the dual model affects recovery of the underlying pointwise multipliers on MNIST. Figure~\ref{fig:mnist_capacity} varies the number of parameters in $\lambda_\phi$ over several orders of magnitude and reports both ranking recovery and the corresponding approximation--generalization decomposition. Increasing dual capacity improves ranking recovery, particularly for the smaller parameterizations, indicating that insufficiently expressive networks are primarily limited by approximation error. As capacity increases further, the gains become progressively smaller and eventually saturate. In contrast to the stronger train--test tradeoff observed in the Flora-Bench experiment, the generalization error on MNIST remains comparatively stable over the range of model sizes considered. This behavior is likely facilitated by the large number of samples available in MNIST, which provides substantially more data for learning increasingly expressive dual functions. This experiment illustrates that the tradeoff between dual expressivity and statistical complexity depends jointly on the capacity of $\Lambda_\Phi$ and on the amount of data available to estimate the multiplier function.

\begin{figure}[t]
    \centering
    \includegraphics[width=\linewidth]{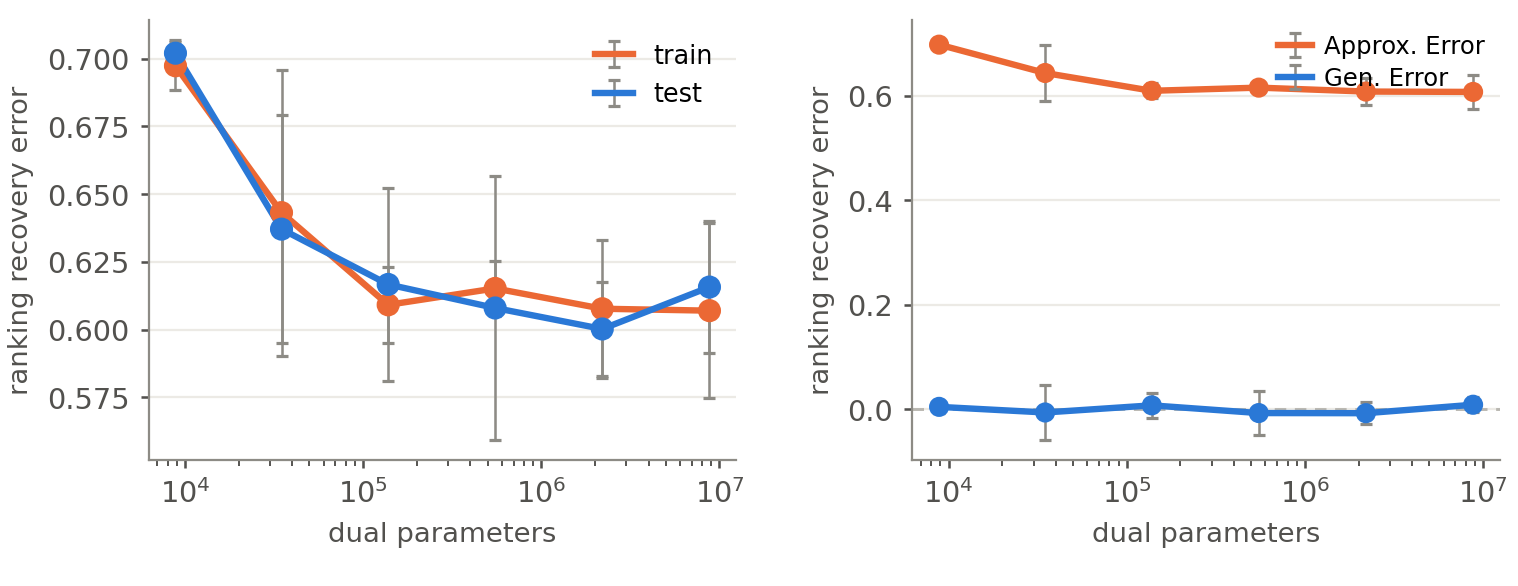}
    \caption{\textbf{Effect of dual parameterization capacity on MNIST.}
    Ranking recovery and error decomposition as the number of parameters in the dual network increases. Increasing capacity primarily reduces approximation error and improves recovery of the pointwise multiplier ranking, while the generalization error remains comparatively stable across model sizes.}
    \label{fig:mnist_capacity}
\end{figure}

\subsection{Additional Sensitivity Results}
\label{app:additional_sensitivity}

\paragraph{Sensitivity on FLORA-Bench.}
Finally, we repeat the direct sensitivity experiment on FLORA-Bench. Figure~\ref{fig:flora_sensitivity} evaluates how the primal objective changes when constraints associated with unseen samples are imposed, grouping the samples according to the magnitude of their predicted multipliers. As the multiplier norm increases, the resulting increase in the optimized objective also tends to grow. The separation is less sharp than in the MNIST experiment reported in the main text, but the overall trend is similar: constraints assigned larger dual values tend to have a greater effect on the primal optimum when enforced. This weaker relationship is consistent with the more heterogeneous structure of FLORA-Bench, where sensitivity can depend jointly on the workflow, the task, and its interaction with the remaining constrained samples. Nevertheless, the observed increase provides direct evidence that the predicted multiplier retains meaningful information about the cost of enforcing previously unseen constraints. Together with the ranking results of Table~\ref{tab:dual_generalization}, this indicates that the learned dual captures both which unseen constraints are important and, to a meaningful extent, how strongly they influence the optimized objective.

\begin{figure}[t]
    \centering
    \includegraphics[width=0.5\linewidth]{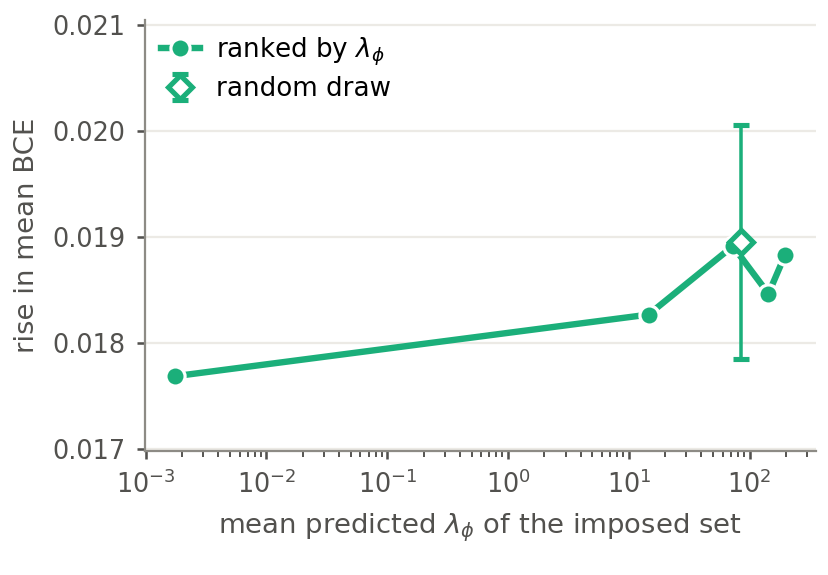}
    \caption{\textbf{Sensitivity estimation on FLORA-Bench.}
    Change in the optimized objective after imposing constraints on unseen samples, grouped according to the magnitude of their predicted multipliers. Sets with larger predicted multiplier values induce larger objective changes, providing an operational validation of the sensitivity information encoded by the learned dual.}
    \label{fig:flora_sensitivity}
\end{figure}

\end{document}